\documentclass[accepted]{uai2022} 

\usepackage[american]{babel}

\usepackage{natbib} 
\usepackage{booktabs} 
\usepackage{tikz} 

\usepackage{xr}
\usepackage{amsfonts}   
\usepackage{amsmath}
\usepackage{amssymb}
\usepackage{mathtools}
\usepackage{amsthm}
\usepackage{thmtools}
\usepackage{threeparttable}
\usepackage{multirow}
\usepackage{algorithm}
\usepackage{algorithmic}

\usepackage{caption}
\usepackage{subcaption}
\usepackage{appendix}

\usepackage{amsmath,amsfonts,bm}

\def\1{\bm{1}}

\def\vtheta{{\bm{\theta}}}
\def\vTheta{{\bm{\Theta}}}
\def\valpha{{\bm{\alpha}}}
\def\vphi{{\bm{\phi}}}

\def\vf{{\bm{f}}}

\def\vh{{\bm{h}}}

\def\vo{{\bm{o}}}

\def\vu{{\bm{u}}}

\def\vx{{\bm{x}}}
\def\vy{{\bm{y}}}
\def\vz{{\bm{z}}}

\DeclareMathAlphabet{\mathsfit}{\encodingdefault}{\sfdefault}{m}{sl}
\SetMathAlphabet{\mathsfit}{bold}{\encodingdefault}{\sfdefault}{bx}{n}

\def\gA{{\mathcal{A}}}

\def\gD{{\mathcal{D}}}

\def\gF{{\mathcal{F}}}

\def\gH{{\mathcal{H}}}

\def\gN{{\mathcal{N}}}
\def\gO{{\mathcal{O}}}

\def\gQ{{\mathcal{Q}}}

\newcommand{\E}{\mathbb{E}}
\newcommand{\Ls}{\mathcal{L}}

\DeclareMathOperator*{\argmax}{arg\,max}
\DeclareMathOperator*{\argmin}{arg\,min}

\DeclareMathOperator{\diag}{diag}

\newtheorem{proposition}{Proposition}
\newtheorem*{proposition*}{Proposition}
\newtheorem*{remark}{Remark}

\newcommand{\rom}[1]{%
  \textup{\uppercase\expandafter{\romannumeral#1}}%
}

\title{Neural Ensemble Search via Bayesian Sampling}

\author[1]{Yao Shu}
\author[1]{Yizhou Chen}
\author[1]{Zhongxiang Dai}
\author[1]{Bryan Kian Hsiang Low}
\affil[1]{%
    Department of Computer Science\\
    National University of Singapore\\
    Singapore
}
  
\begin{document}
\maketitle

\begin{abstract}
Recently, \emph{neural architecture search} (NAS) has been applied to automate the design of neural networks in real-world applications. A large number of algorithms have been developed to improve the search cost or the performance of the final selected architectures in NAS. Unfortunately, these NAS algorithms aim to select only \emph{one single} well-performing architecture from their search spaces and thus have overlooked the capability of \emph{neural network ensemble} (i.e., an ensemble of neural networks with diverse architectures) in achieving improved performance over a single final selected architecture.
To this end, we introduce a novel neural ensemble search algorithm, called \emph{neural ensemble search via Bayesian sampling} (NESBS), to effectively and efficiently select well-performing neural network ensembles from a NAS search space. In our extensive experiments, NESBS algorithm is shown to be able to achieve improved performance over state-of-the-art NAS algorithms while incurring a comparable search cost, thus indicating the superior performance of our NESBS algorithm over these NAS algorithms in practice.
\end{abstract}

\section{Introduction}

Recent years have witnessed a surging interest in designing well-performing architectures for different tasks. These architectures are typically manually designed by human experts, which requires numerous trials and errors during this manual design process and therefore is prohibitively costly. Consequently, the increasing demand for developing well-performing architectures in different tasks makes this manual design infeasible. To avoid such human efforts, \citet{nas} have introduced \emph{neural architecture search} (NAS) to help automate the design of architectures. Since then, a number of NAS algorithms \citep{enas, darts, p-darts} have been developed to improve the search efficiency (i.e., search cost) or the search effectiveness (i.e., generalization performance of their final selected architectures) in NAS.

However, conventional NAS algorithms aim to select only \emph{one single architecture} from their search spaces and have thus overlooked the capability of other candidate architectures from the same search spaces in helping improve the performance achieved by their final selected single architecture. 
That is, \emph{neural network ensembles} are widely known to be capable of achieving an improved performance compared with a single neural network in practice \citep{adanet, mc-dropout, deepens}.
This naturally begs the question: \emph{How to select best-performing neural network ensembles with diverse architectures from a NAS search space in order to improve the performances achieved by existing NAS algorithms?}
To the best of our knowledge, only limited efforts (e.g., \citep{nes}) have been devoted to this problem in the NAS literature. Unfortunately, the \emph{neural ensemble search} (NES) algorithm based on random search or evolutionary algorithm in \citep{nes} requires excessive search costs to select their final neural network ensembles, which will not be affordable in resource-constrained scenarios.

To this end, this paper introduces a novel algorithm, namely \emph{neural ensemble search via Bayesian sampling} (NESBS), to effectively and efficiently select the well-performing neural network ensemble with diverse architectures from a search space.
We firstly represent the search space as a supernet following conventional one-shot NAS algorithms and then use the model parameters inherited from this supernet after its model training to estimate the single-model performances and also the ensemble performance of independently trained architectures (Sec. \ref{sec:oneshot}). Next, since both single-model performances and diverse model predictions affect the final ensemble performance according to \citep{zhou-ensemble}, we propose to use a variational posterior distribution of architectures based on a trained supernet to characterize these two factors, i.e., single-model performances and diverse model predictions (Sec. \ref{sec:posterior}). We then introduce two novel Bayesian sampling algorithms based on the posterior distribution of architectures, i.e., \emph{Monte Carlo sampling} (MC Sampling) and \emph{Stein Variational Gradient Descent with regularized diversity} (SVGD-RD), to effectively and efficiently select ensembles with both competitive single-model performances and compelling diverse model predictions (Sec. \ref{sec:sampling}), which is also guaranteed to be able to achieve impressive ensemble performances \citep{zhou-ensemble}. Lastly, we use extensive experiments to show that our NESBS algorithm is indeed able to select well-performing neural network ensembles effectively and efficiently in practice (Sec.~\ref{sec:exps}).

\section{Related Works \& Background}

\subsection{Neural Architecture Search}
In the literature, many NAS algorithms \citep{amoebanet, nas, nasnet} have been developed to automate the design of well-performing neural architectures. However, these NAS algorithms are inefficient in practice due to their requirement of the independent model training for each candidate architecture in the search space. To reduce such training costs, a supernet has been introduced to represent the search space and also share model parameters among the candidate architectures in the search space \citep{enas}. As a result, only the model training of this supernet is required, which can significantly improve the search efficiency of conventional NAS algorithms. After that, a number of one-shot NAS algorithms based on model parameter sharing \citep{p-darts, sdarts, darts-, darts, snas} have been developed. 
Unfortunately, these algorithms aim to select \emph{only one single architecture} from their search spaces. Thus, the capability of other candidate architectures from the same search spaces in helping improve the performance of their final selected single architecture have been overlooked.

\subsection{Neural Network Ensembles}
Meanwhile, neural network ensembles have been widely applied to improve the performance of a single neural network in different applications \citep{eml}. Over the years, a number of methods have been proposed to construct such neural network ensembles. For example, \citet{mc-dropout} have proposed to use Monte Carlo Dropout to obtain neural network ensembles at test time. Meanwhile, \emph{deep ensembles} (DeepEns) \citep{deepens} adopt neural networks trained with different random initializations to construct ensembles and has achieved impressive performances on various tasks. 
Another line of ensemble works uses the checkpoints obtained during model training to build neural network ensembles \citep{snapshot}. More recently, \citet{nes} have introduced \emph{neural ensemble search} (NES) into NAS area to build well-performing neural network ensembles by selecting diverse architectures from the NAS search space, which has achieved competitive performance even compared with other ensemble methods. Unfortunately, the algorithm presented in \citep{nes} is shown to be prohibitively costly, which will not be affordable in resource-constrained scenarios. To this end, this paper presents a novel NESBS algorithm to advance this line of works (e.g., NES) by achieving state-of-the-art performances for neural network ensembles with diverse architectures while incurring a reduced search cost.

\subsection{Stein Variational Gradient Descent}\label{sec:bg-svgd}
\emph{Stein Variational Gradient Descent} (SVGD) \citep{svgd} is a variational inference algorithm that approximates a target distribution $p(\vx)$ with a simpler density $q^*(\vx)$ in a predefined set $\gQ$ by minimizing the \emph{Kullback-Leibler} (KL) divergence between these two densities:
\begin{equation}\label{eq:min-kl}
    q^* = \argmin_{q \in \gQ} \{\text{KL}(q || p)\triangleq\E_q\left[\log\left(q(\vx) / p(\vx)\right)\right]\} \ .
\end{equation}
Specifically, SVGD represents $q^*(\vx)$ with a set of particles $\{\vx_i\}_{i=1}^{n}$ which are firstly randomly initialized and then iteratively updated with updates $\vphi^*(\vx_i)$ and a step size $\epsilon$:
\begin{equation}\label{eq:svgd-iter}
    \vx_i \leftarrow \vx_i + \epsilon\vphi^*(\vx_i) \quad \text{for}\ i = 1,\ldots,n \ .
\end{equation}
Let $q_{[\epsilon\vphi]}$ denote the distribution of updated particles $\vx' = \vx + \epsilon\vphi(\vx)$. Let $\mathbb{F}$ denote the unit ball of a vector-valued \emph{reproducing kernel Hilbert space} (RKHS) $\gH\triangleq\gH_0 \times \ldots \times\gH_0$ where $\gH_0$ is an RKHS formed by scalar-valued functions associated with a positive definite kernel $k(\vx, \vx')$. The work of \citet{svgd} has shown that \eqref{eq:svgd-iter} can be viewed as functional gradient descent in the RKHS $\gH$ and the optimal $\vphi^*$ in \eqref{eq:svgd-iter} can be obtained by solving the following problem:
\begin{equation}
    \vphi^* = \argmax_{\vphi \in \mathbb{F}}\left\{-\frac{d}{d\epsilon}\text{KL}(q_{[\epsilon\vphi]} || p)\Bigr|_{\epsilon=0}\right\} \ ,
\end{equation}
which yields a closed-form solution: 
\begin{equation}
    \vphi^*(\cdot) = \E_{\vx\sim q}[k(\vx, \cdot)\nabla_{\vx}\log p(\vx) + \nabla_{\vx}k(\vx, \cdot)] \ .
\end{equation}
In practice, \citet{svgd} have approximated the expectation in this closed-form solution with the empirical mean of particles: $\vphi^*(\vx_i) \approx \widehat{\vphi}^*(\vx_i)$ where $\widehat{\vphi}^*(\vx_i)$ is defined as
\begin{equation}\label{eq:svgd-approx}
    \widehat{\vphi}^*(\vx_i) \triangleq \\
    \frac{1}{n}\sum_{j=1}^n k(\vx_j, \vx_i)\nabla_{\vx_j}\log p(\vx_j) + \nabla_{\vx_j} k (\vx_j, \vx_i) \ .
\end{equation}
As revealed in \citep{svgd}, the two terms in the aforementioned closed-form solution take different effects: The first term with $\nabla_{\vx}\log p(\vx)$ favors particles with higher probability density, while the second term pushes the particles away from each other to encourage diversity.

\section{Neural Ensemble Search via Bayesian Sampling}
Contrary to the selection of one single architecture in conventional NAS algorithm, this paper focuses on the problem of selecting a well-performing neural network ensemble with diverse architectures from the NAS search space, i.e., \emph{neural ensemble search} (NES) \citep{nes}. Let $\vf_{\gA}(\vx, \vtheta_{\gA})$ denote the output of an architecture $\gA$ with input data $\vx$ and model parameter $\vtheta_{\gA}$,  $S$ be a set of architectures, $\vTheta_S$ be a set of the corresponding model parameters of these architectures, and $\Ls_{\text{train}}$ and $\Ls_{\text{val}}$ denote the training and validation losses, respectively. Given the ensemble scheme $\gF_S(\vx, \vTheta_S)\triangleq n^{-1}\textstyle\sum_{\gA \in S}\vf_{\gA}(\vx, \vtheta_{\gA})$ with an ensemble size of $|S|=n$,\footnote{We apply this ensemble scheme for simplicity. Other ensemble schemes can also be used in the algorithm of this paper.} NES can be formally framed as
\begin{equation}\label{eq:nes}
\begin{gathered}
    \min_{S} \Ls_{\text{val}}(\gF_S(\vx, \vTheta_S^*)) \\
    \text{s.t.} \; \forall  \vtheta^*_{\gA} \in \vTheta^*_S \quad \vtheta^*_{\gA} = \argmin_{\vtheta_{\gA}} \Ls_{\text{train}}(\vf_{\gA}(\vx, \vtheta_{\gA})) \ .
\end{gathered}
\end{equation}

Unfortunately, \eqref{eq:nes} is challenging to solve mainly due to the following two reasons: (\rom{1}) The enormous number of candidate architectures in the NAS search space  (e.g., ${\sim}10^{25}$ in the DARTS search space \citep{darts}) makes the independent model training of every candidate architecture (i.e., lower-level optimization in \eqref{eq:nes}) unaffordable. (\rom{2}) The ensemble search space is exponentially increasing in the ensemble size $n$: For example, there are ${\sim}m^n$ different ensembles given $m$ diverse architectures. The combinatorial optimization problem (i.e., upper-level optimization in \eqref{eq:nes}) is thus intractable to solve within this huge ensemble search space. Recently, \citet{nes} have attempted to avoid these two problems by sampling a small pool of architectures from the search space for their final ensemble search. Thus, they fail to explore the whole search space and may achieve poor ensemble performances in practice. Moreover, their search cost is still unaffordable due to the independent model training of every architecture in the pool.

To this end, we novelly present the \emph{neural ensemble search via Bayesian sampling} (NESBS) algorithm to solve \eqref{eq:nes} effectively and efficiently. We firstly employ the model parameters inherited from a supernet (i.e., a representation of the NAS search space) after its model training to estimate the single-model performances and also the ensemble performance of independently trained architectures (Sec. \ref{sec:oneshot}). This only requires the model training of the supernet and thus allows us to overcome the aforementioned challenge \rom{1}. We then derive a posterior distribution of architectures to characterize both the single-model performances and the diverse model predictions of candidate architectures in the search space (Sec. \ref{sec:posterior}). Finally, based on this posterior distribution and also the aforementioned ensemble performance estimation, we introduce \emph{Monte Carlo Sampling} (MC Sampling) and \emph{Stein Variational Gradient Descent with regularized diversity} (SVGD-RD) to explore the ensembles in the whole~search space~effectively and efficiently (Sec. \ref{sec:sampling}), which thus allows us to overcome the aforementioned challenge \rom{2}. An overview of our NESBS is in Algorithm~\ref{alg:nes}.

\subsection{Model Training of Supernet}\label{sec:oneshot}

Similar to one-shot NAS algorithms \citep{darts, enas}, we represent NAS search space as a supernet. This then allows us to use the model parameters inherited from this trained supernet to estimate not only the single-model performances but also the ensemble performance of independently trained candidate architectures in the search space. However, in order to realize an accurate and fair estimation of these performances, we need to further ensure that every candidate architecture in the search space is trained for a comparable number of steps, namely, the training fairness among candidate architectures \citep{fairnas}.
To achieve this, in every training step of this supernet, we uniformly randomly sample one single candidate architecture from this supernet for model training (see Fig.~\ref{fig:train}). The training fairness of such a training scheme can then be theoretically guaranteed, as demonstrated in Appendix \ref{sec:proofs}. Moreover, we provide empirical results in Appendix \ref{sec:exp-oneshot} to validate the effectiveness of such performance estimations.

\begin{figure}[t]
\centering
\includegraphics[width=\columnwidth]{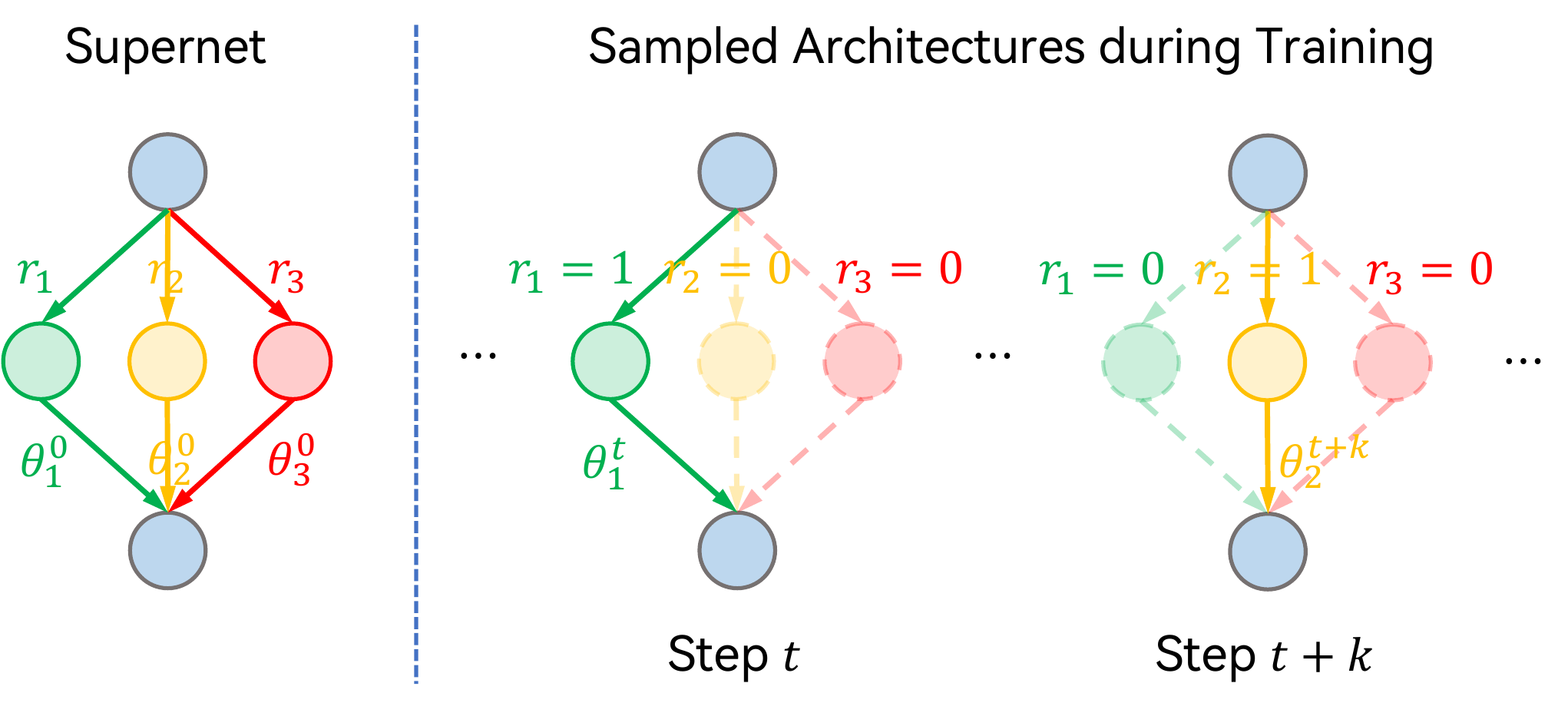}
\caption{An illustration of the model training of~supernet. The supernet here consists of three candidate architectures with $r_i$ indicating the selection of one architecture and $\vtheta_i^{t}$ denoting its model parameters at step $t$. In every training step, only one architecture is uniformly sampled to update its parameters and all other architectures will be ignored.}
\label{fig:train}
\end{figure}

\subsection{Distribution of Architectures}\label{sec:posterior}
It has been demonstrated that both competitive single-model performances and diverse model predictions are required to achieve compelling ensemble performances \citep{zhou-ensemble}. That is, NES algorithms should be capable of selecting architectures with both competitive single-model performances and diverse model predictions to achieve competitive ensemble performances. To realize this, we introduce a posterior distribution of architectures to firstly characterize these two factors. Let $\gD$ denote the validation dataset, and $p(\gA)$ and $p(\gA | \gD)$ denote, respectively, the prior and posterior distributions of a candidate architecture after its model training where $p(\gA)$ follows from a categorical uniform distribution, as required in Sec. \ref{sec:oneshot}. According to the Bayes' theorem, since $p(\gA)$ is uniform and $p(\gD)$ is constant,
\begin{equation}\label{eq:bayes}
\begin{gathered}
    p(\gA|\gD) = p(\gD|\gA)p(\gA)/p(\gD) \propto p(\gD|\gA)
\end{gathered}
\end{equation}
where $p(\gD|\gA)$ (i.e., likelihood) is widely used to represent the single-model performance (i.e., loss) in practice.
So, \eqref{eq:bayes} implies that the posterior distribution $p(\gA|\gD)$ can also characterize the single-model performances of architectures.

Meanwhile, given a $\gamma$-Lipschitz continuous loss function $\Ls(\vf)$, 
the diversity of model predictions (i.e., $\|\vf_{\gA_1}{-}\vf_{\gA_2}\|_2$) can then be lower bounded  based on the Lipschitz continuity of $\Ls(\cdot)$:
\begin{equation}\label{eq:approx}
    \begin{gathered}
        \|\vf_{\gA_1} - \vf_{\gA_2}\|_2 \geq \gamma^{-1}|\Ls(\vf_{\gA_1}) - \Ls(\vf_{\gA_2})| \ .
    \end{gathered}
    \end{equation}
Therefore, \eqref{eq:approx} suggests that in addition to being able to characterize the single-model performances of architectures (i.e., $\Ls(\vf)$), the posterior distribution $p(\gA|\gD)$ can estimate the diversity of model predictions for different architectures (e.g., $\gA_1$ and $\gA_2$) using $|p(\gA_1|\gD) - p(\gA_2|\gD)|$.

However, it is intractable to obtain exact posterior distribution $p(\gA|\gD)$ in the NAS search space. So, we approximate it with a variational distribution $p_{\valpha}(\gA)$ (parameterized by a low-dimensional $\valpha$) that can be optimized via variational inference, i.e., by minimizing the KL divergence between $p_{\valpha}(\gA)$ and $p(\gA|\gD)$. Equivalently, we only need to maximize a lower bound of the log-marginal likelihood (i.e., the \emph{evidence lower bound} (ELBO) \citep{vae}) to get an optimal variational distribution $p_{\valpha^*}(\gA)$:
\begin{equation}\label{eq:elbo}
\begin{gathered}
    \max_{\valpha} \E_{\gA \sim p_{\valpha}(\gA)}\left[\log p(\gD|\gA)\right] - \text{KL}[p_{\valpha}(\gA) || p(\gA)] \ .
\end{gathered}
\end{equation}
Similar to \citep{vae}, a gradient-based optimization algorithm with the reparameterization trick is employed to solve \eqref{eq:elbo} efficiently (see Appendix \ref{sec:setting-posterior}). While \citet{snas} have adopted a similar form to \eqref{eq:elbo} (without the KL term) \emph{during} the model training of the supernet (namely, the \emph{best-response} posterior distribution), our \emph{post-training} posterior distribution is able to not only provide a more accurate characterization of the single-model performances but also contribute to an improved ensemble search performance, as demonstrated in Appendix \ref{sec:exp-best_vs_post}. 

\subsection{Bayesian Sampling}\label{sec:sampling}
To solve \eqref{eq:nes} effectively and efficiently, we finally introduce two novel Bayesian sampling algorithms based on the posterior distribution of architectures in Sec.~\ref{sec:posterior}, i.e., \emph{Monte Carlo sampling} (MC Sampling) and \emph{Stein Variational Gradient Descent with regularized diversity} (SVGD-RD), to sample ensembles with both competitive single-model performances and compelling diversity of model predictions, as required by well-performing ensembles~\citep{zhou-ensemble}.

\begin{figure}[t]
\begin{minipage}{\columnwidth}
\begin{algorithm}[H]
  \caption{NES via Bayesian Sampling (NESBS)}
  \label{alg:nes}
\begin{algorithmic}[1]
  \STATE {\bfseries Input:} Iterations $T$, ensemble size $n$, a supernet
  \STATE Train the supernet to get its tuned parameters $\vtheta^*$
  \STATE Obtain the posterior distribution $p_{\valpha^*}(\gA)$ with \eqref{eq:elbo}
  \FOR{iteration $t=1, \ldots, T$}
  \STATE Sample $S_t$ of size $n$ via Algorithm \ref{alg:mc} or \ref{alg:svgd-rd} 
  \STATE Evaluate estimated $\Ls_{\text{val}}(\gF_{S_t}(\vx, \vTheta_{S_t}^*))$ given $\vtheta^*$
  \ENDFOR
  \STATE Select optimum $S^* = \argmin_{S_t} \Ls_{\text{val}}(\gF_{S_t}(\vx, \vTheta_{S_t}^*))$
\end{algorithmic}
\end{algorithm}
\end{minipage}
\hfill
\begin{minipage}{\columnwidth}
\begin{algorithm}[H]
  \caption{MC Sampling}
  \label{alg:mc}
\begin{algorithmic}[1]
  \STATE {\bfseries Input:} Ensemble size $n$, set $S=\emptyset$, posterior $p_{\valpha^*}(\gA)$
  \FOR{iteration $i=1, \ldots, n$}
  \STATE Sample $\gA_i \sim p_{\valpha^*}(\gA)$
  \STATE $S \leftarrow S \cup \{\gA_i\}$
  \ENDFOR
  \STATE {\bfseries Output:} $S$
\end{algorithmic}
\end{algorithm}
\end{minipage}

\begin{minipage}{\columnwidth}
\begin{algorithm}[H]
  \caption{SVGD-RD}
  \label{alg:svgd-rd}
\begin{algorithmic}[1]
  \STATE {\bfseries Input:} Diversity coefficient $\delta$, ensemble size $n$, iterations $L$, initial particles $\{\vx_i^{(0)}\}_{i=1}^n$, posterior $p_{\valpha^*}(\gA)$, kernel $k(\vx, \vx')$, step size $\{\epsilon_l\}_{l=1}^L$
  \FOR{iteration $l=0, \ldots, L-1$}
  \STATE Evaluate updates $\widehat{\vphi}_l^*(\vx) =\displaystyle\frac{1}{n}\sum_{j=1}^n\nabla_{\vx_j^{(l)}}k(\vx_j^{(l)}, \vx)- \delta \nabla_{\vx}k(\vx_j^{(l)}, \vx) + k(\vx_j^{(l)},\vx)\nabla_{\vx_j^{(l)}}\log p_{\valpha^*}$
  \STATE Update particles $\vx_i^{(l+1)} \leftarrow \vx_i^{(l)} + \epsilon_l\ \widehat{\vphi}_l^*(\vx_i^{(l)})$
  \ENDFOR
  
  \STATE {\bfseries Output:} $S=\{\gA_i\}_{i=1}^n$ derived based on $\{\vx_i^{(L)}\}_{i=1}^n$
\end{algorithmic}
\end{algorithm}
\end{minipage}
\end{figure}

\subsubsection{Monte Carlo Sampling (MC Sampling)}\label{sec:mc}
Given the posterior distribution of architectures in Sec. \ref{sec:posterior}, we firstly propose to use \emph{Monte Carlo sampling} (MC Sampling) to sample a set of architectures from this posterior distribution (Algorithm \ref{alg:mc}). Note that MC Sampling guarantees that architectures with better single-model performances will be sampled (i.e., exploited) with higher probabilities, while architectures with diverse model predictions
can also be sampled (i.e., explored) due to the inherent randomness in the sampling process.
Compared with conventional NAS algorithms that select only one single well-performing architecture from the search space \citep{gdas, snas}, our MC sampling algorithm extends these algorithms by exploring the capability of diverse architectures while preserving its exploitation of architectures with compelling single-model performances.

\begin{figure}[t]
\centering
\includegraphics[width=\columnwidth]{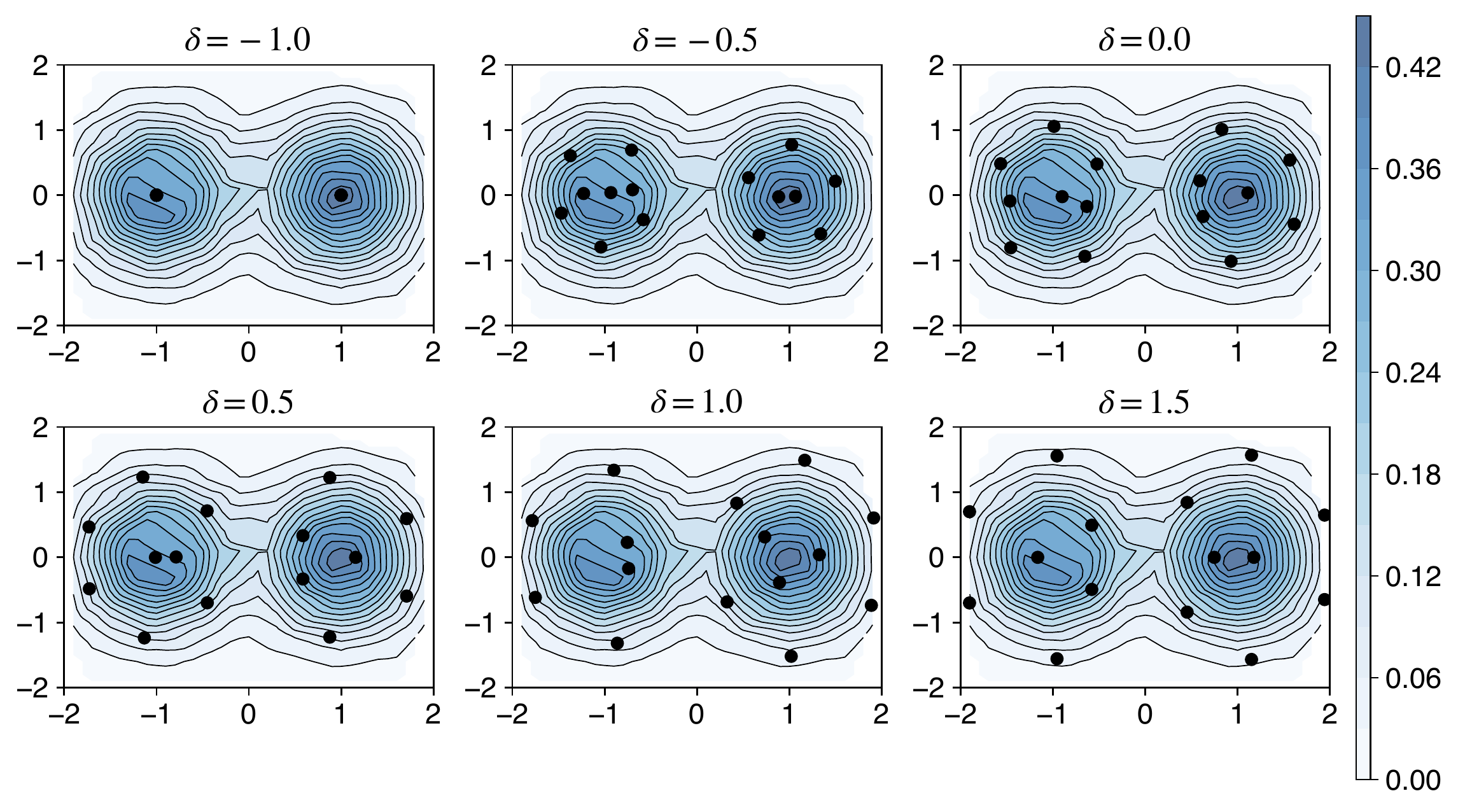}
\caption{Impact of $\delta$ in SVGD-RD. We use contours and dots to denote the density of target distribution and sampled particles, respectively. The target distribution is chosen to be $p(\vx){=}(1/Z)\left[\gN(\vx | \vu_1, \Sigma_1) + \gN(\vx | \vu_2, \Sigma_2)\right]$ where $\vu_1{=}(-1, 0)$, $\vu_2{=}(0, 1)$, $\Sigma_1{=}\Sigma_2{=}\diag((0.25, 0.5))$, and $Z$ denotes the normalization constant. Those sampled particles are obtained from Algorithm \ref{alg:svgd-rd} using $L{=}1000$, $n{=}15$, $\epsilon_l{=}0.1$, and a \emph{radial basis function} (RBF) kernel. Notably, SVGD-RD tends to sample particles with more diverse probability densities as $\delta$ is increased, hence indicating a controllable (via $\delta$) diversity in our SVGD-RD algorithm. Meanwhile, SVGD-RD can consistently sample particles with high probability densities under varying $\delta$.}
\label{fig:diversity}
\end{figure}

\subsubsection{SVGD with Regularized Diversity (SVGD-RD)}\label{sec:svgd-rd}
However, the diversity of sampled architectures using the MC Sampling algorithm above cannot be controlled and hence may lead to poor ensemble search results.
So, in order to achieve a controllable diversity, we resort to \emph{Stein Variational Gradient Descent} (SVGD). Theoretically, SVGD is capable of sampling particles with both large probability density and good diversity where the diversity is explicitly encouraged (i.e., by the second term in \eqref{eq:svgd-approx}). Nonetheless, in practice, the particles sampled by SVGD may still fail to represent the target distribution well owing to the lack of diversity among those sampled particles, as observed in \citep{mp-svgd}. Besides, the diversity of sampled particles in standard SVGD  still cannot be controlled by human experts. 

We hence develop an \emph{SVGD with regularized diversity} (SVGD-RD) sampling algorithm that can achieve a controllable diversity among those sampled particles. We follow the notations from Sec. \ref{sec:bg-svgd}.
In particular, when optimizing the distribution $q^*$ (represented by the $n$ particles $\{\vx^*_i\}_{i=1}^{n}$), we modify the objective in \eqref{eq:min-kl} by adding a term representing the (controllable) diversity among the particles measured by the kernel function $k(\vx, \vx')$:
\begin{equation}\label{eq:min-kl-max-div}
    q^* = \argmin_{q \in \gQ}\text{KL}(q||p) + n\delta\  \E_{\vx,\vx' \sim q} \left[k(\vx, \vx')\right] 
\end{equation}
where $\delta$ is the parameter explicitly controlling the diversity, and $p$ in \eqref{eq:min-kl-max-div} denotes the posterior distribution $p_{\valpha^*}(\gA)$ derived in Sec. \ref{sec:posterior} which we intend to sample from.
Following the work of SVGD, $q^*$ in \eqref{eq:min-kl-max-div} is represented by $\{\vx^*_i\}_{i=1}^{n}$ denoting our final selected neural network ensemble that can achieve both competitive single-model performances (i.e., large probability density) and also diverse model predictions. Proposition \ref{prop:svgdrd-update} below provides one possible update rule for the particles $\{\vx_i\}_{i=1}^{n}$ to optimize \eqref{eq:min-kl-max-div} (see its proof in Appendix \ref{sec:proofs}).
Finally, Algorithm \ref{alg:svgd-rd} summarizes the details of our SVGD-RD algorithm and Appendix \ref{sec:setting-svgd} provides its optimization details in practice. After obtaining those optimal particles $\{\vx^*_i\}_{i=1}^{n}$ in our SVGD-RD algorithm, we then apply these particles to derive the architectures in our final selected ensembles (see details in Appendix \ref{sec:setting-svgd}).

\begin{proposition}\label{prop:svgdrd-update}
    Given the proximal operator $\normalfont \text{prox}_h(\vy)=\argmin_{\vz}h(\vz)+1/2\|\vz-\vy\|_2^2$, by applying proximal gradient method \citep{proximal} and proper approximation, \eqref{eq:min-kl-max-div} can be optimized via the following updates of the particles $\{\vx_i\}_{i=1}^{n}$:
    \begin{equation*}
    \begin{array}{l}
        \displaystyle\vx_i \leftarrow \vx_i + \frac{1}{n} \sum_{j=1}^n k(\vx_j, \vx_i)\nabla_{\vx_j}\log p(\vx_j) \\
        \displaystyle\qquad\qquad\qquad\quad\ + \nabla_{\vx_j} k (\vx_j, \vx_i) - \delta\nabla_{\vx_i} k (\vx_j, \vx_i)\ .
    \end{array}
    \end{equation*}
\end{proposition}

Compared with MC Sampling, our SVGD-RD algorithm provides a controllable trade-off between the single-model performances and the diverse model predictions.
On the one hand, the minimization of the KL divergence term in \eqref{eq:min-kl-max-div} encourages the selection of architectures with competitive single-model performances by favoring particles with high probability densities, as shown by Proposition \ref{prop:exploitation} below (its proof is in Appendix \ref{sec:proofs}).\footnote{Although Proposition \ref{prop:exploitation} is only applicable in the case of $n=1$, our SVGD-RD is still capable of sampling particles with high probability densities when $n>1$, as validated in Fig.~\ref{fig:diversity}.}
On the other hand, the maximization of the scaled distance $-n\delta\ \E_{\vx,\vx' \sim q} \left[k(\vx, \vx')\right]$ among the sampled particles leads to a controllable diversity (via $\delta$) among these sampled particles and also a controllable diversity of the probability densities among these particles (see Fig.~\ref{fig:diversity}), which also implies a controllable diversity of the model predictions, as suggested in Sec.~\ref{sec:posterior}.

\begin{proposition}\label{prop:exploitation}
Let $p$ be a target density and $k(\vx,\vx')=c$ for every $\vx=\vx'$ where $c$ is a constant. For any $\delta \in \mathbb{R}$, our SVGD-RD algorithm is equivalent to the maximization of the density $p$ w.r.t.~$\vx$ in the case of $n=1$.
\end{proposition}

\section{Experiments}\label{sec:exps}
\subsection{Search in NAS-Bench-201}\label{sec:nasbench}

\begin{table*}[t]
\caption{Comparison of architectures selected by different NAS and ensemble (search) algorithms in NAS-Bench-201 with ensemble size $n=3$. Test errors are reported with the mean and standard error of three independent trials and our search costs are evaluated on a single Nvidia $1080$Ti GPU.  Results marked by $\dagger$ are reported by \citet{nasbench201}.}
\centering
\resizebox{0.87\textwidth}{!}{
\begin{threeparttable}
\begin{tabular}{lcccc}
\toprule
\multirow{2}{*}{\textbf{Architecture(s)}} & \multicolumn{3}{c}{\textbf{Test Error} (\%)} &
\multirow{2}{2cm}{\textbf{ Search Cost} (GPU Hours)} \\
\cmidrule(l){2-4}
& \textbf{CIFAR-10} & \textbf{CIFAR-100} & \textbf{ImageNet-16-200} &  \\
\midrule
& \multicolumn{4}{c}{\textbf{Manual design}} \\
ResNet$^{\dagger}$ \citep{resnet} & 6.03 & 29.14 & 56.37 & - \\
\midrule
& \multicolumn{4}{c}{\textbf{NAS algorithms}} \\
ENAS$^{\dagger}$ \citep{enas} & 45.70$\pm$0.00 & 84.39$\pm$0.00 & 83.68$\pm$0.00 & 3.7 \\
DARTS$^{\dagger}$ (2nd) \citep{darts} & 45.70$\pm$0.00 & 84.39$\pm$0.00 & 83.68$\pm$0.00 & 8.3 \\
GDAS$^{\dagger}$ \citep{gdas} & 6.49$\pm$0.13 & 29.39$\pm$0.26 & 58.16$\pm$0.90 & 8.0 \\
SETN$^{\dagger}$ \citep{setn} & 13.81$\pm$4.63 & 43.13$\pm$7.77 & 68.10 $\pm$4.07 & 8.6 \\
RSPS$^{\dagger}$ \citep{rsps} & 12.34$\pm$1.69 & 41.67$\pm$4.34 & 68.86$\pm$3.88 & 2.1 \\
\midrule
& \multicolumn{4}{c}{\textbf{Ensemble (search) algorithms}} \\
DeepEns \citep{deepens} & 5.75 & 25.27 & 54.70 & - \\
NES-RS \citep{nes} & 5.83$\pm$0.33 & 25.58$\pm$0.84 & 54.34$\pm$1.67 & 5.1 \\
\midrule
& \multicolumn{4}{c}{\textbf{Our ensemble search algorithm}} \\
NESBS (MC Sampling) & 5.76$\pm$0.25 & 25.39$\pm$0.69 & \textbf{53.47}$\pm$1.75 & \textbf{1.1} \\
NESBS (SVGD-RD) & 5.92$\pm$0.07 & \textbf{25.00}$\pm$0.17 & \textbf{52.68}$\pm$0.35 & \textbf{1.2} \\
\bottomrule
\end{tabular}
\end{threeparttable}
}
\label{tab:nasbench201}
\end{table*}

To verify the effectiveness and efficiency of our NESBS algorithm, we firstly compare it with other well-known NAS and ensemble (search) algorithms in NAS-Bench-201 \citep{nasbench201}. Table~\ref{tab:nasbench201} summarizes the results. Table~\ref{tab:nasbench201} shows that ensemble (search) algorithms, including our NESBS, consistently achieve improved generalization performance over conventional NAS algorithms. This is because ensemble (search) algorithms will select neural network ensembles whereas NAS algorithms will select only one single architecture. Moreover, it has been widely verified that model ensembles generally outperform a single machine learning model in practice \citep{zhou-ensemble}. In addition, our NESBS algorithm outperforms other ensemble (search) baseline (i.e., DeepEns and NES-RS), especially on large-scale datasets (i.e., CIFAR-100 \citep{cifar} and ImageNet-16-200 \citep{imagenet-16-120}) while incurring less search costs than NES-RS, which thus implies the superior performance of our NESBS over these ensemble (search) baselines. Even on a small-scale dataset (i.e., CIFAR-10), our NESBS can also achieve comparable search results to DeepEns and NES-RS.  Interestingly, our NESBS algorithm is even able to incur reduced search costs than conventional NAS algorithms. This is likely because more training epochs have been used in these NAS algorithms, whereas a small number of training epochs can already contribute to well-performing results for our NESBS algorithm.

\subsection{Search in The DARTS Search Space}\label{sec:darts}
We further demonstrate the superior search effectiveness and efficiency of our NESBS by comparing it with other NAS and ensemble (search) baselines in a larger search space (i.e.,  DARTS \citep{darts} search space) using both classification and adversarial defense tasks on CIFAR-10/100 or ImageNet~\citep{imagenet}. We follow Appendix \ref{sec:setting-training} to evaluate the final neural network ensembles selected by our NESBS algorithm with ensemble size $n=3$, $T=5$, and optimization details in Appendix~\ref{sec:setting-exp}.

\begin{table*}[t]
\caption{Comparison of different image classifiers on CIFAR-10/100. Results of MC DropPath are from a drop~probability of $0.01$ and our search costs are evaluated on Nvidia $1080$Ti.}
\centering
\resizebox{0.87\textwidth}{!}{
\begin{threeparttable}
\begin{tabular}{lcccccc}
\toprule
\multirow{2}{*}{\textbf{Architecture(s)}} & \multicolumn{2}{c}{\textbf{Test Error} (\%)} &
\multicolumn{2}{c}{\textbf{Params} (M)} &
\multirow{2}{1.8cm}{\textbf{Search Cost} (GPU Days)} &
\multirow{2}{*}{\textbf{Search Method}} \\
\cmidrule(l){2-3} \cmidrule(l){4-5} 
& \textbf{C10} & \textbf{C100} & \textbf{C10} & \textbf{C100} & \\
\midrule 
& \multicolumn{6}{c}{\textbf{NAS algorithms}} \\

NASNet-A \citep{nasnet} & 2.65 & - & 3.3 & - & 2000 & RL\\
AmoebaNet-A \citep{amoebanet} & 3.34 & 18.93 & 3.2 & 3.1 & 3150 & evolution\\
PNAS \citep{pnas} & 3.41 & 19.53 & 3.2 & 3.2 & 225 & SMBO\\
ENAS \citep{enas} & 2.89 & 19.43 & 4.6 & 4.6 & 0.5 & RL\\
DARTS \citep{darts} & 2.76 & 17.54 & 3.3 & 3.4 & 1 & gradient\\
GDAS \citep{gdas} & 2.93 & 18.38 & 3.4 & 3.4 & 0.3 & gradient \\

P-DARTS \citep{p-darts} & 2.50 & - & 3.4 & - & 0.3 & gradient \\
DARTS- (avg) \citep{darts-} & 2.59 & 17.51 & 3.5 & 3.3 & 0.4 & gradient \\
SDARTS-ADV \citep{sdarts} & 2.61 & - & 3.3 & - & 1.3 & gradient \\
\midrule
& \multicolumn{6}{c}{\textbf{Ensemble (search) algorithms}} \\
MC DropPath (ENAS) & 2.88 & 16.83 & 3.8$^\ddagger$ & 3.9$^\ddagger$ & - & - \\

DeepEns (ENAS) & 2.49 & 15.04 & 3.8$^\ddagger$ & 3.9$^\ddagger$ & - & - \\
DeepEns (DARTS) & 2.42 & 14.56 & 3.3$^\ddagger$ & 3.4$^\ddagger$ & - & - \\

NES-RS$^{\sharp}$ \citep{nes} & 2.50 & 15.24 & 3.0$^\ddagger$ & 3.1$^\ddagger$ & 0.7 & greedy \\

\midrule
& \multicolumn{6}{c}{\textbf{Our ensemble search algorithm}} \\
NESBS (MC Sampling) & \textbf{2.41} & 14.70 & 3.8$^\ddagger$ & 3.9$^\ddagger$ & \textbf{0.2} & sampling \\
NESBS (SVGD-RD) & \textbf{2.36} & \textbf{14.55} & 3.7$^\ddagger$ & 3.8$^\ddagger$ & \textbf{0.2} & sampling \\
\bottomrule
\end{tabular}
\begin{tablenotes}\footnotesize
    \item[$\ddagger$] Reported as the averaged parameter size of the architectures in a neural network ensemble.
    \item[$\sharp$] Obtained from a pool of size $50$, in which every architecture is uniformly randomly sampled from the DARTS search spaces and then trained independently for $50$ epochs following the evaluation settings in Appendix~\ref{sec:setting-training}.
\end{tablenotes}
\end{threeparttable}
}
\label{tab:accuracy-cifar}
\end{table*}

\begin{table}[!t]
\caption{Comparison of image classifiers on ImageNet. The ensemble size is set to $n=3$ for NES-RS and NESBS.}
\centering
\resizebox{\columnwidth}{!}{
\begin{threeparttable}
\begin{tabular}{lcccc}
\toprule
\multirow{2}{*}{\textbf{Architecture(s)}} & \multicolumn{2}{c}{\textbf{Test Error} (\%)} &
\multirow{2}{1.0cm}{\textbf{Params}} &
\multirow{2}{0.6cm}{\textbf{$+\times$}} \\
\cmidrule(l){2-3}
& \textbf{Top-1} & \textbf{Top-5} & (M) & (M) \\

\midrule
\multicolumn{5}{c}{\textbf{NAS algorithms}} \\
NASNet-A  & 26.0 & 8.4 & 5.3 & 564 \\
AmoebaNet-A & 25.5 & 8.0 & 5.1 & 555 \\
PNAS & 25.8 & 8.1 & 5.1 & 588 \\
DARTS & 26.7 & 8.7 & 4.7 & 574 \\
GDAS & 26.0 & 8.5 & 5.3 & 581 \\
P-DARTS & 24.4 & 7.4 & 4.9 & 557 \\
SDARTS-ADV & 25.2 & 7.8 & 5.4 & 594 \\
\midrule
\multicolumn{5}{c}{\textbf{Ensemble (search) algorithm}} \\
NES-RS & 23.4 & 6.8 & 3.9 & 432 \\
\midrule
\multicolumn{5}{c}{\textbf{Our ensemble search algorithm}} \\
NESBS (MC Sampling) & \textbf{22.3} & \textbf{6.2} & 4.6 & 522 \\
NESBS (SVGD-RD) & \textbf{22.3} & \textbf{6.1} & 4.9 & 562 \\
\bottomrule
\end{tabular}
\end{threeparttable}
}
\label{tab:accuracy-imagenet}
\end{table}

\paragraph{Ensemble for classification.}
Table \ref{tab:accuracy-cifar} summarizes the comparison of classification performances on CIFAR-10/100. Similar to the results in Sec. \ref{sec:nasbench}, ensemble (search) algorithms, including our NESBS, are generally able to achieve improved generalization performances over conventional NAS algorithms, which thus justifies the essence of ensemble (search) algorithms for improved performance.
Notably, even compared with other ensemble baselines such as MC DropPath (i.e., developed following Monte Carlo Dropout \citep{mc-dropout}) and DeepEns, our NESBS is still able to achieve improved performances. Since these ensemble baselines are orthogonal to our NESBS, they can be integrated into our NESBS for further performance improvement in real-world applications. More importantly, our algorithm outperforms NES-RS by achieving both improved search effectiveness (lowest test errors) and efficiency (lowest search costs). Furthermore, our NESBS even incurs comparable search costs compared with the most efficient NAS algorithms (e.g., GDAS, P-DARTS), which also highlights the efficiency of our NESBS. Similar results on ImageNet can be achieved by our NESBS as shown in Table~\ref{tab:accuracy-imagenet}. \footnote{Following the convention of NAS and ensemble search algorithms in Table~\ref{tab:accuracy-imagenet}, the ensembles selected by our NESBS are also searched on CIFAR-10 and then transferred to ImageNet.}

\paragraph{Ensemble for adversarial defense.}\label{sec:exp-adversarial}
Ensemble methods have already been shown to be an essential and effective defense mechanism against adversarial attacks~\citep{ensemble-for-defense}.
Specifically, an adversarial attacker can only use \emph{a single model} randomly sampled from an ensemble to generate the adversarial examples, whereas the ensemble method defends against adversarial attacks (i.e., makes its predictions) using \emph{all models} in this ensemble. Ensemble methods can defend against the adversarial attacks in such a setting because the generated adversarial examples using only one single model are unlikely to fool all models in an ensemble.
More details are provided in Appendix \ref{sec:setting-adversarial}.
Table \ref{tab:adversarial} summarizes the comparison of adversarial defense among ensemble (search) algorithms on CIFAR-10/100 under different
white-box adversarial attacks, including the \emph{Fast Gradient Signed Method} (FGSM) attack \cite{fgsm}, the \emph{Projected Gradient Descent} (PGD) attack \cite{pgd}, the \emph{Carlini Wagner} (CW) attack \cite{cw}, and the AutoAttack~\citep{autoattack}.
Table \ref{tab:adversarial} shows that ensemble (search) algorithms are indeed able to significantly improve the performance of adversarial defense, i.e., the test accuracies in the \emph{Defense} column are consistently higher than the ones in \emph{Attack} column. More importantly, even under different white-box adversarial attacks, our NESBS algorithm can generally achieve improved defense performances (i.e., higher test accuracy in the \emph{Defense} columns) than other baselines including DeepEns and NES-RS. These results thus further support the effectiveness of our NESBS over existing ensemble (search) algorithms.
Besides, even regarding the adversarial robustness of the single models in an ensemble, the architectures selected by our NESBS are also more advanced (i.e., by achieving higher test accuracy in the \emph{Attack} columns) than well-known architectures such as RobNet~\citep{robnet} and DARTS.

\begin{table*}[t]
\caption{Comparison of adversarial defense among different ensemble (search) algorithms on CIFAR-10/100 under white-box adversarial attacks.
The \emph{Attack} and \emph{Defense} columns denote the test \emph{accuracy} under the attack using a single model randomly sampled from an ensemble and the defense using the whole ensemble, respectively.
Each result reports the mean and standard deviation of test accuracies for $3$ rounds of the attack-defense process with an ensemble size of $n=3$.}
\newcommand{\ms}{\phantom{-}}
\renewcommand\multirowsetup{\centering}
\centering
\resizebox{\textwidth}{!}{
\begin{tabular}{l*{8}{c}}
\toprule
\multirow{2}{*}{\textbf{Method}} & 
\multicolumn{2}{c}{\textbf{FGSM}} &  
\multicolumn{2}{c}{\textbf{PGD-40}} &
\multicolumn{2}{c}{\textbf{CW}} &
\multicolumn{2}{c}{\textbf{AutoAttack}}
\\
\cmidrule(l){2-3} \cmidrule(l){4-5} \cmidrule(l){6-7} \cmidrule(l){8-9} 
&
Attack (\%) & Defense (\%) &
Attack (\%) & Defense (\%) & 
Attack (\%) & Defense (\%) & 
Attack (\%) & Defense (\%)  \\
\midrule 
& \multicolumn{8}{c}{\textbf{On CIFAR-10 Dataset}} \\
DeepEns & - & - & - & - & - & - & - & - \\
$\quad \hookrightarrow$ RobNet-free & 66.62$\pm$0.32 & 85.25$\pm$0.39 & 41.81$\pm$0.80 & 77.48$\pm$0.67 & $\;\;$5.74$\pm$1.41 & 86.53$\pm$0.50 & 21.35$\pm$0.33 & 45.51$\pm$0.15 \\
$\quad \hookrightarrow$ ENAS & 77.85$\pm$0.58 & 87.94$\pm$0.21 & 59.51$\pm$1.13 & 86.57$\pm$0.15 & 31.36$\pm$1.20 & 85.20$\pm$0.77 & 31.71$\pm$0.72 & 50.96$\pm$0.07 \\
$\quad \hookrightarrow$ DARTS & 76.79$\pm$0.80 & 88.21$\pm$0.14 & 57.71$\pm$1.65 & 82.02$\pm$0.10 & 26.90$\pm$1.37 & 82.46$\pm$0.35 & 29.97$\pm$1.17 & 49.67$\pm$0.14 \\
NES-RS & 79.19$\pm$1.39 & 89.32$\pm0.27$ & 65.59$\pm$2.11 & 85.22$\pm$0.41 & 37.20$\pm$4.62 & 86.75$\pm$0.88 & 35.00$\pm$1.15 & 53.80$\pm$0.14 \\
\cmidrule(l){1-9}
NESBS (MC Sampling) & 78.75$\pm$1.29 & 89.15$\pm$0.08 & 63.60$\pm$1.87 & 85.35$\pm$0.31 & \textbf{37.71}$\pm$1.97 & \textbf{86.86}$\pm$0.66 & \textbf{36.02}$\pm$0.64 & \textbf{56.90}$\pm$0.17 \\
NESBS (SVGD-RD) & 79.12$\pm$0.61 & \textbf{89.86}$\pm$0.33 & 65.53$\pm$1.56 & 85.37$\pm$0.38 & \textbf{38.27}$\pm$1.27 & 86.00$\pm$1.10 & \textbf{37.55}$\pm$0.68 & \textbf{57.15}$\pm$0.20\\
\midrule
& \multicolumn{8}{c}{\textbf{On CIFAR-100 Dataset}} \\
DeepEns & - & - & - & - & - & - & - & - \\
$\quad \hookrightarrow$ RobNet-free & 36.47$\pm$0.25 & 61.39$\pm$0.30 & 18.18$\pm$0.47 & 52.61$\pm$0.13 & 2.36$\pm$0.13 & 69.44$\pm$0.04 & $\;\;$7.31$\pm$0.35 & 24.56$\pm$0.33 \\
$\quad \hookrightarrow$ ENAS & 46.40$\pm$0.37 & 64.94$\pm$0.27 & 28.87$\pm$0.27 & 56.79$\pm$0.25 & 9.60$\pm$0.30 & 69.43$\pm$0.44 & 11.53$\pm$0.47 & 27.01$\pm$0.27 \\
$\quad \hookrightarrow$ DARTS & 46.98$\pm$0.57 & 65.38$\pm$0.23 & 28.78$\pm$0.74 & 57.10$\pm$0.04 & 9.73$\pm$0.43 & 70.15$\pm$0.29 & 11.20$\pm$0.40 & 26.86$\pm$0.36 \\
NES-RS & 47.10$\pm$1.46 & 65.33$\pm0.36$ & 30.68$\pm$1.66 & 58.80$\pm$0.80 & 9.96$\pm$1.45 & 70.24$\pm$0.33 & 12.01$\pm$0.93 & 27.49$\pm$0.34 \\
\cmidrule(l){1-9}
NESBS (MC Sampling) & \textbf{50.69}$\pm$1.58 & \textbf{67.63}$\pm$0.05 & \textbf{33.37}$\pm$0.42 & \textbf{60.36}$\pm$0.62 & \textbf{15.64}$\pm$2.83 & \textbf{71.25}$\pm$1.27 & \textbf{13.11}$\pm$1.16 & \textbf{29.87}$\pm$1.17 \\
NESBS (SVGD-RD) & \textbf{51.47}$\pm$0.40 & \textbf{66.66}$\pm$0.13 & \textbf{35.02}$\pm$0.37 & \textbf{59.96}$\pm$0.18 & \textbf{16.72}$\pm$0.61 & \textbf{69.88}$\pm$0.16 & \textbf{14.62}$\pm$0.55 & \textbf{31.07}$\pm$0.33 \\
\bottomrule
\end{tabular}
}
\label{tab:adversarial}
\end{table*}

\subsection{Single-model Performances and Diverse Model Predictions}\label{sec:trade-off}
We demonstrate that the effectiveness of our NESBS results from its ability to achieve a good trade-off between the single-model performances and the diversity of model predictions. 
We firstly quantitatively compare the single-model performances (measured by the \emph{averaged test error} (ATE) of the models in an ensemble) and the diversity of model predictions (measured by the \emph{pairwise predictive disagreement} (PPD) of an ensemble \citep{disagreement}) achieved by different ensemble (search) algorithms on CIFAR-10/100. 
We further qualitatively visualize their single-model performances and diverse model predictions using a histogram of the ATE of the models in their ensembles and a t-SNE \citep{t-sne} plot of their model predictions, respectively.

\begin{table}[t]
\caption{Quantitative comparison of the single-model performances (measured by ATE (\%), smaller is better) and the diversity of model predictions (measured by PPD (\%), larger is better) achieved by different ensemble (search) algorithms with an ensemble size of $3$ on CIFAR-10/100.}
\renewcommand\multirowsetup{\centering}
\centering
\resizebox{\columnwidth}{!}{
\begin{tabular}{l *{4}{c}}
\toprule
\multirow{2}{*}{\textbf{Method}} & 
\multicolumn{2}{c}{\textbf{C10}} &
\multicolumn{2}{c}{\textbf{C100}} \\
\cmidrule(l){2-3} \cmidrule(l){4-5}
& ATE & PPD & ATE & PPD \\
\midrule 
MC DropPath (DARTS) & 2.71 & 0.39 & 16.68 & 2.63 \\
DeepEns (DARTS) & \textbf{2.69} & 2.08 & \textbf{16.18} & 12.45  \\
NES-RS  & 2.87 & 2.29 & 17.20 & \textbf{14.14} \\
\midrule
NESBS (MC Sampling)  & 2.80 & \textbf{2.57}  & 16.70 & 13.84\\
NESBS (SVGD-RD) & 2.78 & 2.27 & 16.50 & 13.16 \\
\bottomrule
\end{tabular}
}
\label{tab:perf_vs_diver}
\end{table}

Table \ref{tab:perf_vs_diver} and Fig.~\ref{fig:perf_vs_diver} present the results of our quantitative and qualitative comparisons, respectively. Compared with the ensemble baselines of MC DropPath and DeepEns, our NESBS is capable of enjoying a larger diversity of model predictions while preserving competitive single-model performances. Meanwhile, compared with the ensemble search baselines of NES-RS, our algorithm can achieve improved single-model performances while maintaining comparably diverse model predictions. These results suggest that our NESBS is able to select ensembles achieving a better trade-off between the single-model performances and the diversity of model predictions among these baselines, which is known to be an important criterion for well-performing ensembles \citep{zhou-ensemble}. Thus, Table \ref{tab:perf_vs_diver} and Fig.~\ref{fig:perf_vs_diver} provide empirical justifications for the improved effectiveness of NESBS.

\begin{figure}[t]
\centering
\begin{tabular}{cc}
    \hspace{-6mm}\includegraphics[width=0.525\columnwidth]{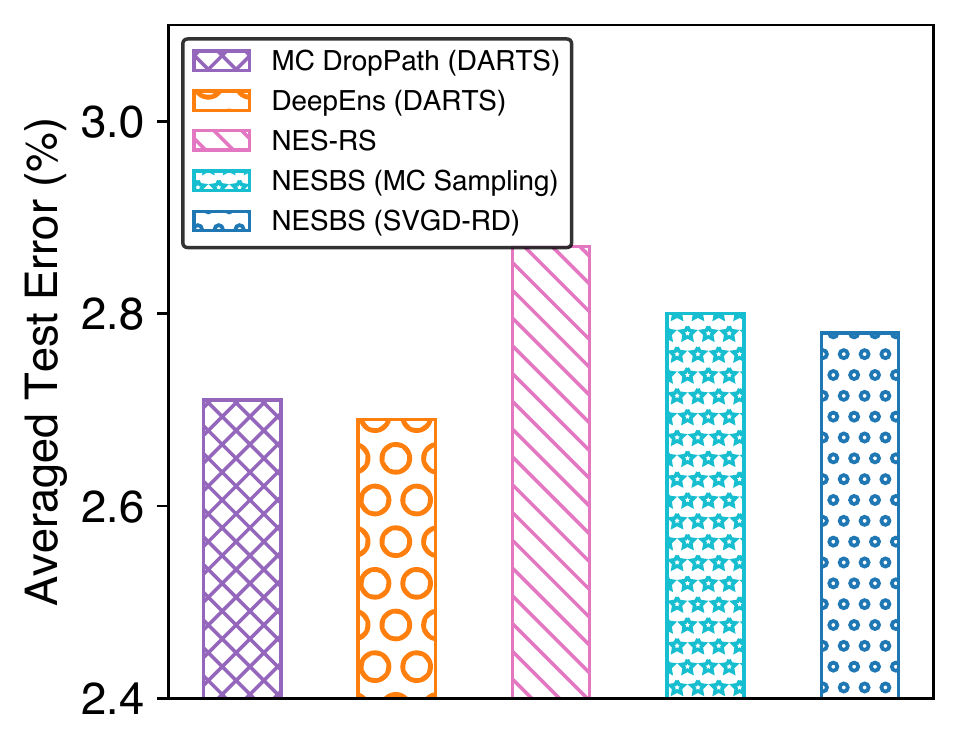} & \hspace{-6mm}\includegraphics[width=0.485\columnwidth]{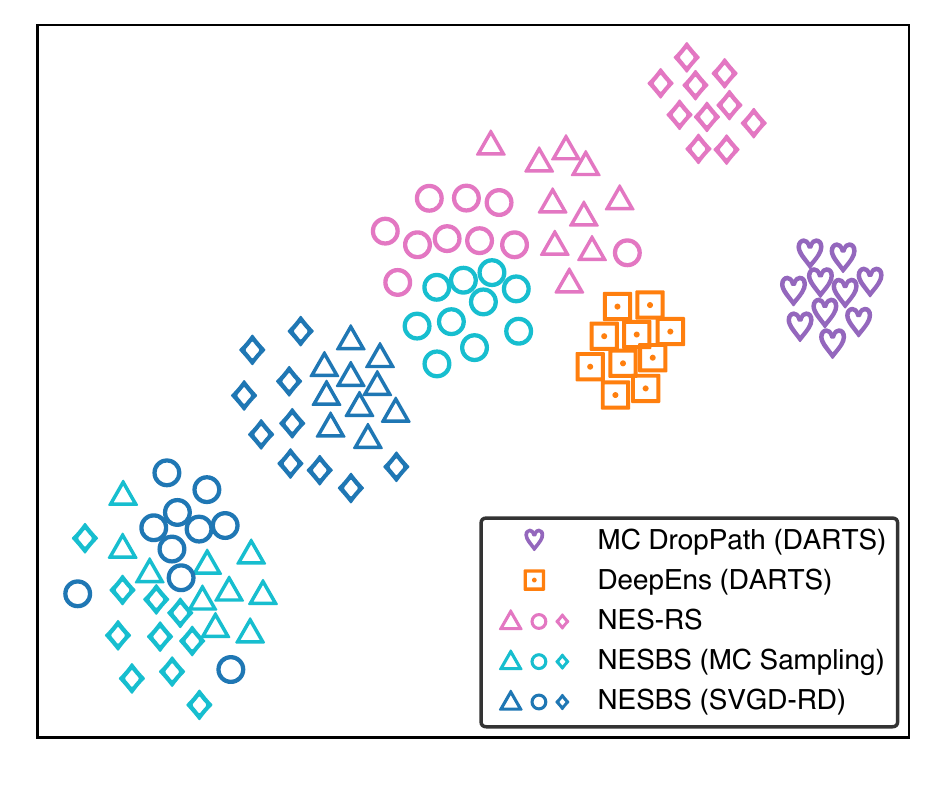} \\
    {\hspace{-3mm} (a) Single-model performances} & \hspace{-6mm} {(b) Diverse predictions}
\end{tabular}
\caption{Qualitative comparison of (a) the single-model performances and (b) the diverse model predictions achieved by different ensemble (search) algorithms with an ensemble size of $n=3$ on CIFAR-10. Each architecture in (b) is independently evaluated for ten times to visualize their model predictions, which follows from DeepEns.}
\label{fig:perf_vs_diver}
\end{figure}

\section{Conclusion}
This paper presents a novel neural ensemble search algorithms, called NESBS, that can effectively and efficiently select well-performing neural network ensembles with diverse architectures from a NAS search space. Our extensive experiments have shown that NESBS is able to achieve improved performances while preserving a comparable search cost compared with conventional NAS algorithms. Moreover, even compared with other ensemble (search) baselines (e.g., DeepEns and NES-RS), our NESBS is also capable of enjoying boosted search effectiveness and efficiency, which further suggests the superior performance of our NESBS in practice.

\begin{acknowledgements} 
This research/project is supported by A*STAR under its RIE$2020$ Advanced Manufacturing and Engineering (AME) Programmatic Funds (Award A$20$H$6$b$0151$).

\end{acknowledgements}

\bibliography{shu_406}

\newpage
\appendix
\begin{appendices}
\section{Proofs}\label{sec:proofs}

\begin{proposition*}{\textbf{\normalfont (Training fairness in the supernet.)}}\label{prop:fair}
Let $T$ and $T_{\gA}$ denote the number of steps applied to train the supernet and candidate architecture $\gA$ in the search space of size $N$, by uniformly randomly sampling a single architecture from this search space for the model training in each step, we have
\begin{equation*}
    \begin{gathered}
        \normalfont\text{Pr}(\lim_{T\rightarrow\infty} T_{\gA_i} / T = \lim_{T\rightarrow\infty} T_{\gA_j} / T) = 1 \quad \forall i,j \in \{1, \cdots, N\} \ .
    \end{gathered}
    \end{equation*}
\end{proposition*}
\begin{proof}
Let random variable $X_i^t \in \{0,1\}$ denote the selection of candidate architecture $\gA_i$ at training step $t$ under our sampling scheme in the proposition above. For any $t>0$ and $i,j \in [N]$, random variable $X_i^t - X_j^t$ can achieve following possible assignments and probabilities (denoted by $p$):
\begin{equation}
\begin{gathered}
    X_i^t - X_j^t = 
    \left\{\begin{array}{rcl}
             +1, & & p = 1/N \\
             0, & & p = (N-2)/N \\
             -1, & & p = 1/N
    \end{array}\right. \ .
\end{gathered}
\end{equation}
Consequently, $\E[X_i^t - X_j^t] = 0$. According to the strong law of large numbers, we further have
\begin{equation}
\begin{gathered}
    \text{Pr}(\lim_{T\rightarrow\infty} T^{-1}\sum_{t=1}^T X_i^t - X_j^t = 0) = 1 \ .
\end{gathered}
\end{equation}
Note that 
\begin{equation}
\begin{gathered}
    T^{-1}(T_{\gA_i} - T_{\gA_j}) = T^{-1}\sum_{t=1}^T X_i^t - X_j^t \ .
\end{gathered}
\end{equation}
We thus can complete this proof by
\begin{equation}
\begin{aligned}
    &\text{Pr}(\lim_{T\rightarrow\infty} T^{-1}(T_{\gA_i} - T_{\gA_j}) = 0) = 1 \ .
\end{aligned}
\end{equation}
\end{proof}

\paragraph{Proof of Proposition \ref{prop:svgdrd-update}.}
As particles $\{\vx_i\}_{i=1}^n$ of size $n$ are applied to approximate the density $q$ in our SVGD-RD, the second term (i.e., the controllable diversity term) in our \eqref{eq:min-kl-max-div} can then be approximated using these particles as
\begin{equation}
\begin{aligned}
    n\delta \E_{\vx,\vx' \sim q} \left[k(\vx, \vx')\right] &\approx  \delta/n \sum_{i=1}^n\sum_{j=1}^n k(\vx_i, \vx_j) \triangleq \sum_{i=1}^n h(\vx_i) \ ,
\end{aligned}
\end{equation}
where $h(\vx)\triangleq\delta/n\sum_{j=1}^n k (\vx, \vx_j)$. We take $\vx_j$ in $k(\vx, \vx_j)$ as a constant for the approximation above. Consequently, we have
\begin{equation}\label{eq:equal-grad}
\begin{aligned}
    \nabla_{\vx_k}\sum_{i=1}^n h(\vx_i) = \nabla_{\vx_k}h(\vx_k) \ .
\end{aligned}
\end{equation}

Let $\vx_i^{+} \triangleq \vx_i + \epsilon\vphi^*(\vx_i)$ ($\forall i\in\{1,\cdots,n\}$) denote the functional gradient decent in the RKHS $\gH$ to minimize the KL divergence term in our \eqref{eq:min-kl-max-div}. Based on \eqref{eq:equal-grad} above, given proximal operator $\text{prox}_h(\vx^{+})=\argmin_{\vy} h(\vy)+1/2\|\vy-\vx^{+}\|_2^2$, by using proximal gradient method \cite{proximal}, our \eqref{eq:min-kl-max-div} can then be optimized via the following update to each particle $\vx_i$:
\begin{equation}
    \vx_i \leftarrow \text{prox}_h(\vx_i^{+})=\argmin_{\vy} h(\vy)+1/2\|\vy-\vx_i^{+}\|_2^2 \ .
\end{equation}

According to the \textit{Karush-Kuhn-Tucker} (KKT) conditions, the local optimum $\vy^*$ of this proximal operator satisfies
\begin{equation}
\begin{aligned}
    \text{prox}_h(\vx_i^{+}) = \vy^* = \vx_i^{+} - \nabla_{\vy^*} h(\vy^*)\ .
\end{aligned}
\end{equation} 
When $h(\cdot)$ is convex, this local optimum is also a global optimum. As (16) is intractable to solve given a complex $h(\cdot)$, we approximate $h(\vy^*)$ with its first-order Taylor expansion, i.e., $h(\vy^*) \approx h(\vx_i) + \nabla_{\vx_i} h(\vx_i)(\vy^* - \vx_i)$ and achieve following approximation:
\begin{equation}
\begin{aligned}
    \text{prox}_h(\vx_i^{+}) &\approx \vx_i^{+} - \nabla_{\vx_i} h(\vx_) \\
    &\approx \vx_i + \epsilon\vphi^*(\vx_i) - \nabla_{\vx_i} h(\vx_i) \\
    &\approx \vx_i + \epsilon\vphi^*(\vx_i) - \delta/n \textstyle\sum_j^n \nabla_{\vx_i} k(\vx_i, \vx_j) \ .
\end{aligned}
\end{equation}

Given the approximation $\vphi^*(\vx_i) \approx \widehat{\vphi}^*(\vx_i)$ and the definition of $\widehat{\vphi}^*(\vx_i)$ in \eqref{eq:svgd-approx}, we complete our proof by
\begin{equation}
\begin{aligned}
    \vx_i \leftarrow \vx_i &+ 1/n \textstyle\sum_{j=1}^n k(\vx_j, \vx_i)\nabla_{\vx_j}\log p(\vx_j) \\
    & + \nabla_{\vx_j}k(\vx_j, \vx_i) - \delta \nabla_{\vx_i} k(\vx_j, \vx_i) \ .
\end{aligned}
\end{equation}

\paragraph{Proof of Proposition \ref{prop:exploitation}.}
Notably, since $k(\vx,\vx')=c$ when $\vx=\vx'$, we will achieve a constant $k(\vx, \vx)$ for any particle $\vx$ in the case of $n=1$, which can be ignored in our SVGD-RD for any $\delta \in \mathbb{R}$. In light of this, our SVGD-RD in the case of $n=1$ degenerates into standard SVGD. Consequently, to prove Proposition \ref{prop:exploitation}, we only need to consider SVGD in the case of $n=1$.

Considering SVGD in the case of $n=1$, we can frame the density $q$ represented by a single particle $\vx'$ as 
\begin{equation}
\begin{aligned}
    q(\vx) = 
    \left\{\begin{array}{rcl}
             1 & & \vx = \vx' \\
             0 & & \vx \neq \vx'
    \end{array}\right. \ .
\end{aligned}
\end{equation}
The KL divergence between $q(\vx)$ and the target density $p(\vx)$ can then be simplified as
\begin{equation}
\begin{aligned}
    \text{KL}(q \| p) &= \E_{q(\vx)}[\log(q(\vx)/p(\vx))] = -\log p(\vx') \ .
\end{aligned}
\end{equation}
Finally, standard SVGD in the case of $n=1$ obtain its optimal particle by optimizing the following problem:
\begin{equation}
\begin{aligned}
    q^* &= \argmin_q \text{KL}(q \| p) \\
    & = \argmin_{\vx'} \{-\log p(\vx')\} \\
    & = \argmax_{\vx'} p(\vx') \ ,
\end{aligned}
\end{equation}
which finally concludes the proof.

\begin{remark}
\emph{
In practice, this $k(\vx,\vx)=c$ can be well satisfied, such as the radial basis function (RBF) kernel that we have applied in our experiments.
}
\end{remark}

\section{Experimental Settings}\label{sec:setting-exp}
\subsection{The DARTS Search Space}\label{sec:setting-search-space}
In the DARTS \citep{darts} search space, each candidate architecture consists of a stack of $L$ cells, which can be represented as a directed acyclic graph (DAG) of $N$ nodes denoted by $\{z_0, z_1, \dots, z_{N-1}\}$. Among these $N$ nodes in a cell, $z_0$ and $z_1$ denote the input nodes produced by two preceding cells, and $z_N$ denotes the output of a cell, which is the concatenation of all intermediate nodes, i.e., from $z_2$ to $z_{N-1}$. As in the work of \citet{darts}, to select the best-performing architectures, we need to select their corresponding cells, including the normal and reduction cell. We refer to the DARTS paper for more details.
In practice, this search space is conventionally represented as a supernet stacked by 8 cells (6 normal cells and 2 reduction cells) with initial channels of 16.

\subsection{Model Training of Supernet}\label{sec:setting-one-shot}
Following \citep{pc-darts}, we apply a partial channel connection with $K=2$ in the model training of the supernet, which allows us to accelerate and reduce the GPU memory consumption during this model training. We split the standard training dataset of CIFAR-10 into two piles in our ensemble search: 70\% randomly sampled data is used in the model training of the supernet, and the rest is used to obtain the posterior distribution of neural architectures in Sec. \ref{sec:posterior} and also the final selected ensembles in Sec. \ref{sec:sampling}. To achieve not only a fair but also a sufficient model training for every candidate architecture, we apply \textit{stochastic gradient descent} (SGD) with epoch 50, learning rate cosine scheduled from 0.1 to 0, momentum 0.9, weight decay $3\times10^{-4}$ and batch size 128 in the model training of the supernet, where only a single candidate architecture is uniformly randomly sampled from this supernet in every training step.

\subsection{Posterior Distribution}\label{sec:setting-posterior}
\paragraph{Variational posterior distribution.}
Following \citep{snas}, the variational posterior distribution of architectures is represented as $p_{\valpha}(\gA)$ parameterized by $\valpha$. Specifically, within the search space demonstrated in our Appendix \ref{sec:setting-search-space}, each intermediate nodes $z_i$ is the output of one selected operation $o \sim p_{\valpha_i}(o)$ using the inputs from its proceeding nodes or cells, where $\gO$ is a predefined operation set for our search. Specifically, given $\valpha_i=(\alpha_i^{o_1} \cdots \alpha_i^{o_{|\gO|}})$, $p_{\valpha_i}(o)$ can be represented as 
\begin{equation}\label{eq:discrete-dist}
    p_{\valpha_i}(o) = \frac{\exp(\alpha_i^o/\tau)}{\sum_{o \in \gO}\exp(\alpha_i^o/\tau)} \ ,
\end{equation}
where $\tau$ denotes the softmax temperature, which is usually set to be 1 in practice. Based on this defined probability for each intermediate node $z_i$, our variational posterior distribution can be framed as
\begin{equation}
    p_{\valpha}(A) = \prod_{i=2}^{N-2}p_{\valpha_i}(o) \ .
\end{equation}
More precisely, this representation is applied for single-path architecture with identical cells. We use it to ease our representation. For double-path architectures consisting of two different cells (i.e., normal and reduction cell), e.g., the candidate architecture in the DARTS search space, a similar representation can be obtained.

\paragraph{Optimization details.}
To optimize \eqref{eq:elbo}, we firstly relax our variational posterior distribution to be differentiable using the Straight-Through (ST) Gumbel-Softmax \citep{concrete, gumbel-softmax} with the reparameterization trick. More precisely, we propose a variant of ST Gumbel-Softmax outputting the double-path architectures in the DARTS search space. Then, we use stochastic gradient-based algorithms to optimize \eqref{eq:elbo} efficiently. In each optimization step, we sample one neural architecture from the distribution $p_{\valpha}(\gA)$ to estimate $\E_{\gA \sim p_{\valpha}(\gA)}\left[\log p(\gD|\gA)\right]$ (i.e., the commonly used Cross-Entropy loss). In practice, we use Adam \citep{adam} with learning rate 0.01, $\beta_1=0.9$, $\beta_2=0.999$ and weight decay $3\times10^{-4}$ to update our variational posterior distribution $p_{\valpha}(\gA)$ for 20 epochs. 

\subsection{SVGD of Regularized Diversity}\label{sec:setting-svgd}
\paragraph{Continuous relaxation of variational posterior distribution.}
Notably, SVGD \citep{svgd} and also our SVGD-RD is applied for continuous distribution. Unfortunately, the variational posterior distribution $p_{\valpha}(\gA)$ is discrete due to a discrete search space. To apply SVGD-RD, we firstly relax this discrete posterior into its continuous counterpart using a mixture of Gaussian distribution. Specifically, we represent each operation $o\in\gO$ in \eqref{eq:discrete-dist} into a one-hot vector $\vh_o$. By introducing the random variable $\vo_i\in\mathbb{R}^{|\gO|}$ and multi-variate normal distribution $\gN(\vo_i | \vh_o, \Sigma)$ into our relaxation, our relaxed posterior distribution of neural architectures can be framed as
\begin{equation}\label{eq:relax-dist}
\begin{gathered}
    \widehat{p}_{\valpha}(\gA) = \prod_{i=2}^{N-2} 1/Z_i\sum_{o \in \gO} p_{\valpha_i}(o) \gN(\vo_i | \vh_o, \Sigma) \ ,
\end{gathered}
\end{equation}
where $Z_i$ denotes the normalization constant. Given the sampled particle $\vx^*=(\cdots \vo_i^* \cdots)$ in SVGD-RD, the final selected architecture can then be derived using the determination of each selected operation $o_i^*$, i.e., 
\begin{equation}
    o_i^* = \argmin_{o \in \gO} \|\vo_i^* - \vh_o\|_2 \ .
\end{equation}

\paragraph{Optimization details.} 
Since \citet{svgd} have demonstrated that SVGD is able to handle unnormalized target distributions, the normalization constant in \eqref{eq:relax-dist} can then be ignored in our SVGD-RD algorithm. In practice, the covariance matrix $\Sigma$ in \eqref{eq:relax-dist} is set to an identity matrix scaled by $|\gO|$. Besides, the parameter $\delta$ is optimized as a hyper-parameter via grid search or Bayesian Optimization \citep{bo} within the range of $[-2, 1]$ in practice. To obtain well-performing particles in our SVGD-RD algorithms efficiently, we apply SGD using the gradient provided in Sec. \ref{sec:svgd-rd} with a \textit{radial basis function} (RBF) kernel on randomly initialized particles for $L{=}1000$ iterations under a learning rate of 0.1 and a momentum of 0.9.

\subsection{Evaluation on Benchmark Datasets}\label{sec:setting-training}
\paragraph{Evaluation on CIFAR-10/100.} 
We apply the same constructions in DARTS \citep{darts} for our final performance evaluation on CIFAR-10/100: The final selected architectures consist of 20 cells, and 18 of them are identical normal cells, with the rest being the identical reduction cell. An auxiliary tower with a weight of 4 is located at the 13-th cell of the final selected architectures. The final selected architecture is then trained using stochastic gradient descent (SGD) for 600 epochs with a learning rate cosine scheduled from 0.025 to 0, momentum 0.9, weight decay $3\times10^{-4}$, batch size 96 and initial channels 36. Cutout \citep{cutout}, and a scheduled DropPath, i.e., linearly decayed from 0.2 to 0,  are employed to achieve SOTA generalization performance.

\paragraph{Evaluation on ImageNet.}
Following \citep{darts}, the architectures evaluated on ImageNet consist of 14 cells (12 identical normal cells and 2 identical reduction cells). To meet the requirement of evaluation under the mobile setting (less than 600M multiply-add operations), the number of initial channels for final selected architectures are conventionally set to 44. We adopt the training enhancements in \cite{darts,p-darts,sdarts}, including an auxiliary tower of weight 0.4 and label smoothing. Following P-DARTS \cite{p-darts} and SDARTS-ADV \cite{sdarts}, we train the selected architectures from scratch for 250 epochs using a batch size of 1024 on 8 GPUs, SGD optimizer with a momentum of 0.9 and a weight decay of $3\times10^{-5}$. The learning rate applied in this training is warmed up to 0.5 for the first 5 epochs and then decreased to zero linearly.

\subsection{Adversarial Defense}\label{sec:setting-adversarial}
Adversarial attack intends to find a small change for each input such that this input with its corresponding small change will be misclassified by a model. As ensemble is known to be a possible defense against such adversarial attacks \citep{ensemble-for-defense}, we also examine the effectiveness of our NESBS algorithm by comparing the model robustNESBS achieved by our algorithms to other ensemble and ensemble search algorithms under various benchmark adversarial attacks. To the best of our knowledge, we are the first to examine the advantages of ensemble search algorithms in defending against adversarial attacks.

In this experiment, two processes are required, i.e., \emph{attack} and \emph{defense}. The \emph{attack} process is a typical white-box attack scenario: Only a single model (randomly sampled from an ensemble) is attacked by an attacker, and this process will be repeated for $n$ rounds given an ensemble of size $n$ in order to accurately measure the improvement of model robustNESBS induced by an ensemble. In each round, a different model from this ensemble is selected to be attacked. 
The \emph{defense} process is then applied using neural network ensembles, i.e., neural network ensembles will make predictions based on those perturbed images produced by the aforementioned attacker. Corresponding to the \emph{attack} process, we also need to repeat this defense process for $n$ rounds. In fact, such an adversarial defense setting is reasonably practical when only a single model from an ensemble is required to be publicly available for model producers.

We apply the following attacks in our experiment: The \emph{Fast Gradient Signed Method} (FGSM) attack \cite{fgsm}, the \emph{Projected Gradient Descent} (PGD) attack \cite{pgd}, the \emph{Carlini Wagner} (CW) attack \cite{cw} and the AutoAttack~\citep{autoattack}. In both the FGSM attack and the PGD attack, we impose a $L_\infty$ norm constrain of $0.01$. The step size and the number of iterations in the PGD attack are set to $0.008$ and $40$, respectively. We adopt the same configurations of the CW attack under a $L_2$ norm constrain in \citep{cw}: We set the confidence constant, the range of constant $c$, the number of binary search steps, and the maximum number of optimization steps to $0$, $[0.001,10]$, $3$, and $50$, respectively; we then adopt Adam \citep{adam} optimizer with learning rate $0.01$ and $\beta_1=0.9$, $\beta_2=0.999$ in its search process. Besides, we adopt the same configuration of AutoAttack from \citep{autoattack}.

\section{Complementary Results}

\subsection{Ensemble Performance Estimation}\label{sec:exp-oneshot}
\begin{table}
\renewcommand\multirowsetup{\centering}
\centering
\begin{tabular}{lcccc}
\toprule
\textbf{Metric} & 
\textbf{$n=1$} & 
\textbf{$n=3$} & 
\textbf{$n=5$} & 
\textbf{$n=7$} \\
\midrule 
Spearman & 0.65 & 0.33 & 0.40 & $-$0.12 \\
Pearson & 0.82 & 0.45 & 0.45 & $-$0.16 \\
Agreement-30\% & 33\% & 20\% & 31\% & 25\% \\
\bottomrule
\end{tabular}
\caption{The correlation between the estimated and true performances of candidate architectures and their ensembles in the DARTS search space on CIFAR-10.}
\label{tab:rank}
\end{table}

As shown in Sec. \ref{sec:oneshot}, we apply the model parameters inherited from a trained supernet to estimate the performance of candidate architectures as well as their ensembles in our NESBS algorithm. We therefore use the following three metrics to measure the effectiveness of such estimation in the DARTS search space: the Spearman's rank order coefficient between the estimated and true performances, the Pearson correlation coefficient between the estimated and true performances, and the percentage of architectures achieving both Top-$k$ estimated performance and Top-$k$ true performances (named the Agreement-$k$). Since the evaluation of the true performances is prohibitively costly, we randomly sample 10 architectures of diverse estimated performances from the DARTS search space for this experiment. Notably, based on these 10 architectures, there are hundreds of possible ensembles under the ensemble size of 3, 5, 7, which we believe is sufficiently large to validate the effectiveness of our performance estimations. To obtain the true performance of candidate architectures as well as their ensembles, we train these architectures independently for 100 epochs following the settings in Appendix \ref{sec:setting-training}. 

Table \ref{tab:rank} summarizes the results. Notably, the estimated and true performances are shown to be positively correlated in the case of $n{=}1,3,5$ by achieving relatively high Spearman and Pearson coefficients as well as a  high agreement in these cases. Although the coefficients are low when the ensemble size is larger (i.e., $n{=}7$), the estimated and true performances are still capable of achieving a reasonably good agreement in this case. 
Based on these results, we argue that our estimated ensemble performance is informative and effective for our ensemble search. This effectiveness can also be supported by the competitive search results achieved by our NESBS in Sec. \ref{sec:darts}.

\begin{figure}[t]
\centering
\resizebox{0.75\columnwidth}{!}{
\begin{tabular}{cc}
    \includegraphics[width=0.36\textwidth]{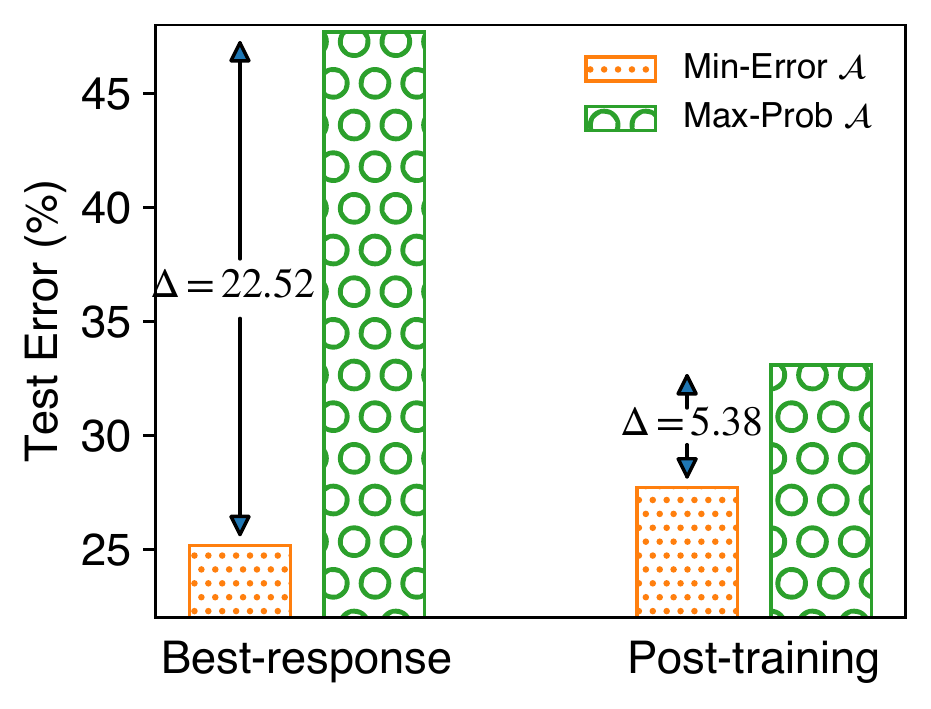}
\end{tabular}
}
\caption{The comparison of performance discrepancy with the post-training and best-response posterior distribution on CIFAR-10. This performance discrepancy is measured by the gap of test error between the best-performing architecture (i.e., the architecture with the smallest test error) and the maximal-probability architecture (i.e., the architecture with the largest probability in the corresponding posterior distribution) in the DARTS search space.}
\label{fig:best_vs_post}
\end{figure}

\subsection{Post-training vs.~Best-response Posterior Distribution}\label{sec:exp-best_vs_post}
To examine the advantages of our post-training posterior distribution, we compare it with its best-response counterpart applied in \citep{gdas, snas}. While our post-training posterior distribution is obtained \emph{after} the model training of the supernet, the best-response posterior distribution is updated \emph{during} the model training of the supernet. We refer to \citep{gdas, snas} for more details about this best-response posterior distribution. We follow the optimization details in Appendix \ref{sec:setting-one-shot} and \ref{sec:setting-posterior} to obtain these two posterior distributions.

\paragraph{More accurate characterization of single-model performances using post-training posterior distribution.}
We firstly compare the characterization of single-mode performance using these two posterior distributions by examining the performance discrepancy between their best-performing architecture (i.e., the architecture achieving the smallest test error) and maximal-probability architecture (i.e.,  the architecture achieving the largest probability in the corresponding posterior distribution) in the search space. In this experiment, the performance discrepancy is measured by the gap of test error achieved by the best-performing architecture and the maximal-probability architecture using the model parameters inherited from the supernet.

Figure \ref{fig:best_vs_post} illustrates the comparison. The results show that our post-training posterior distribution enjoys a smaller performance discrepancy, suggesting that our post-training posterior distribution is able to provide a more accurate characterization of the single-model performances. Interestingly, the best-response counterpart contributes to the best-performing architecture with a lower test error than our post-training posterior distribution, which should result from the Matthew Effect as justified in \citep{dropnas}. Specifically, well-performing architectures contribute to the frequent selections of these architectures for their model training during the optimization of the best-response posterior distribution. This will finally result in unfair model training in the search space and therefore the inaccurate characterization of single-model performances. Notably, we need a more accurate characterization of single-mode performance in this paper, as shown in Sec. \ref{sec:posterior}. Therefore, our post-training posterior distribution should be more suitable than its best-response counterpart in our ensemble search.

\paragraph{Improved performance of selected ensembles using post-training posterior distribution.}
We then compare the final ensemble test performance achieved by our NESBS algorithm using the post-training posterior distribution and its best-response counterpart on CIFAR-10 with the ensemble size of $n=3$. To obtain the final ensemble performance, we train each architecture in an ensemble for 100 epochs following the settings in Appendix \ref{sec:setting-training}. Table \ref{tab:best_vs_post} summarizes the results. Notably, our post-training posterior distribution is shown to be capable of contributing to an improved ensemble performance than its best-response counterpart, which further demonstrates the advantages of applying the post-training posterior distribution in our ensemble search.

\begin{table}[t]
\renewcommand\multirowsetup{\centering}
\centering
\resizebox{\columnwidth}{!}{
\begin{tabular}{lcc}
\toprule
\textbf{Method} & 
\textbf{Best-response} & 
\textbf{Post-training} \\
\midrule 
NESBS (MC Sampling) & 4.74 & \textbf{4.54}$_{\Delta{=}0.20}$ \\
NESBS (SVGD-RD) & 4.81 & \textbf{4.48}$_{\Delta{=}0.33}$\\
\bottomrule
\end{tabular}
}
\caption{The comparison of true ensemble test error (\%) on CIFAR-10 achieved by our NESBS algorithm using the post-training posterior distribution and its best-response counterpart with an ensemble size of $n{=}3$. We use $\Delta$ to denote the improved generalization performance achieved by our post-training posterior distribution.}
\label{tab:best_vs_post}
\end{table}

\subsection{Effectiveness and Efficiency}\label{sec:efficient-and-effective}

\begin{figure}[t]
\centering
\includegraphics[width=\columnwidth]{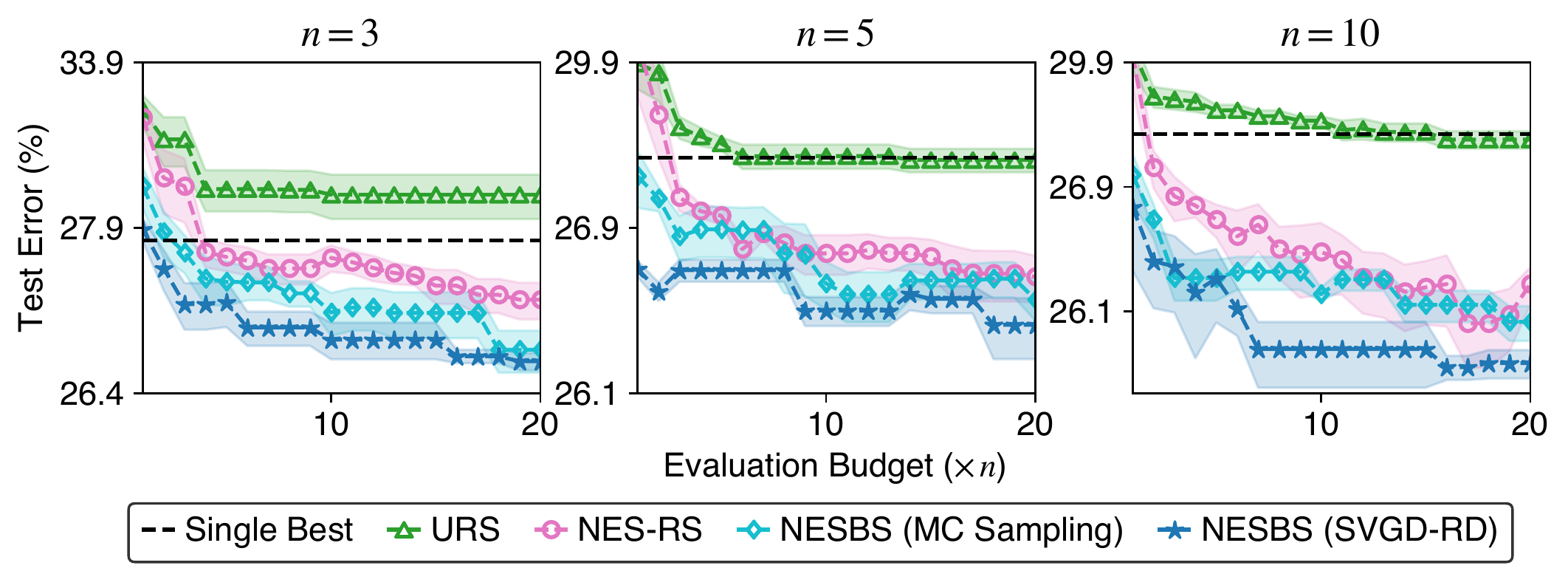}
\caption{The comparison of search effectiveness (test error of ensembles in the $y$-axis) and efficiency (evaluation budget in the $x$-axis) for different ensemble search algorithms under varying ensemble size $n$. The single best baseline refers to the single best architecture achieving the lowest test error in the search space. The $y$-axis is shown in log-scale to ease visualization. Note that the test error for each algorithm is reported with the mean and standard error of five independent trials.
}
\label{fig:search-efficiency}
\end{figure}

As justified in Sec \ref{sec:sampling}, both our MC Sampling and SVGD-RD algorithms can sample neural architectures with competitive single-model performances and diverse model predictions, which are known to be the criteria for well-performing ensembles \citep{zhou-ensemble}. 
To further demonstrate that our algorithms are capable of selecting well-performing ensembles effectively and efficiently based on this sampling property,
we compare our NESBS algorithm, including NESBS (MC Sampling) and NESBS (SVGD-RD), with the following ensemble search baselines on CIFAR-10 \citep{cifar} in the DARTS \citep{darts} search space: (a) Uniform random sampling which we refer to as URS, and (b) NES-RS \citep{nes}. That is, we only replace the Bayesian sampling in our NESBS algorithm with these two different sampling/selection algorithms in this experiment and we keep using the model parameters inherited from a supernet to estimate the single-model and ensemble performances of architectures (including the test errors).
The detailed experimental settings are in Appendix \ref{sec:setting-exp}.

Figure \ref{fig:search-efficiency} illustrates the search results. Note that both NES-RS and our NESBS are able to achieve lower test errors than the single best-performing architecture in the search space. These results therefore demonstrate that these two ensemble search algorithms are indeed capable of achieving improved performance over conventional NAS algorithms that select only one single architecture from the search space. More importantly, given the same evaluation budgets, our NESBS algorithm consistently achieves lower test errors than URS and NES-RS, indicating the superior search effectiveness achieved by our NESBS algorithm. 
Meanwhile, our NESBS algorithm requires fewer evaluation budgets than URS and NES-RS to achieve comparable test errors, which also suggests that our algorithm is more efficient than URS and NES-RS.
Interestingly, compared with MC Sampling, SVGD-RD can consistently produce improved search effectiveness and efficiency, which likely results from its controllable trade-off between the single-model performances and the diverse model predictions as justified in Sec.~\ref{sec:sampling}.
Overall, these results have well justified the effectiveness and efficiency of our NESBS algorithm.

\begin{figure}[t]
\centering
\includegraphics[width=\columnwidth]{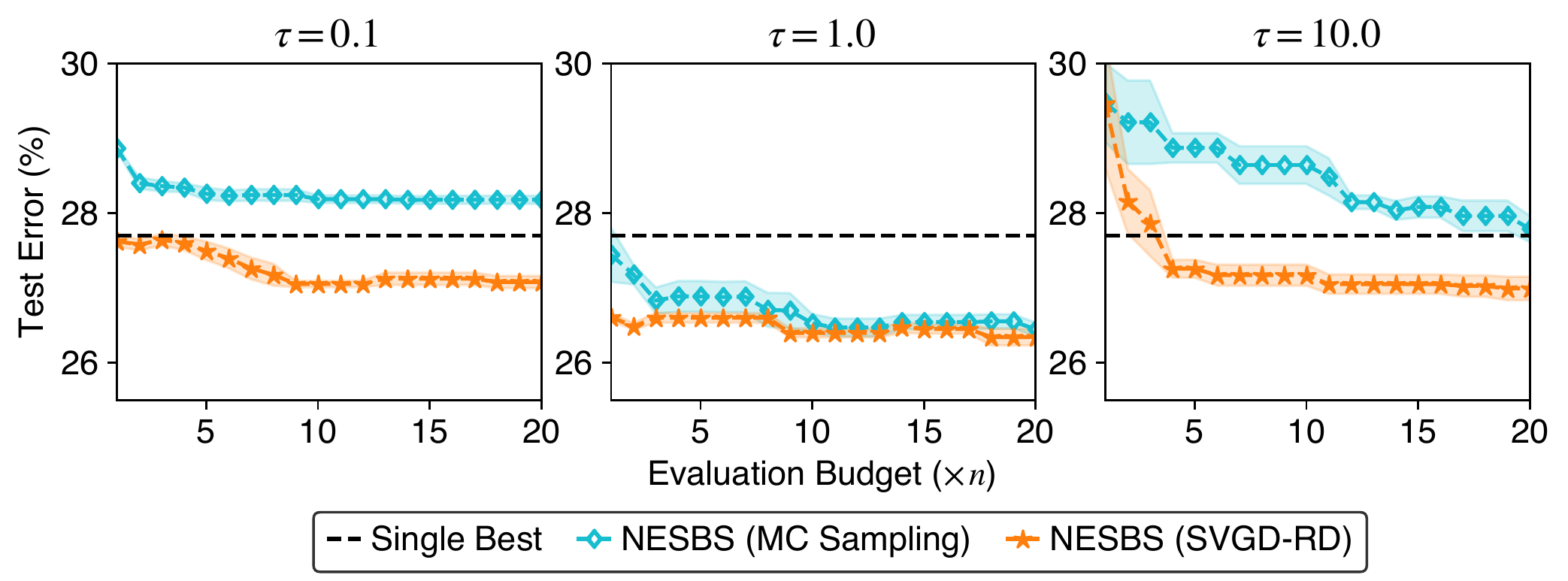}
\caption{
The comparison of search effectiveness (test error of ensembles in the $y$-axis) and efficiency (evaluation budget in the $x$-axis) between our NESBS (MC Sampling) and NESBS (SVGD-RD) algorithm under varying softmax temperature $\tau$. The single best baseline refers to the single best architecture achieving the lowest test error in the search space. Each test error is reported with the mean and standard error of five independent trials.}
\label{fig:robust-tau}
\end{figure}

\begin{figure}[t]
\centering
\includegraphics[width=\columnwidth]{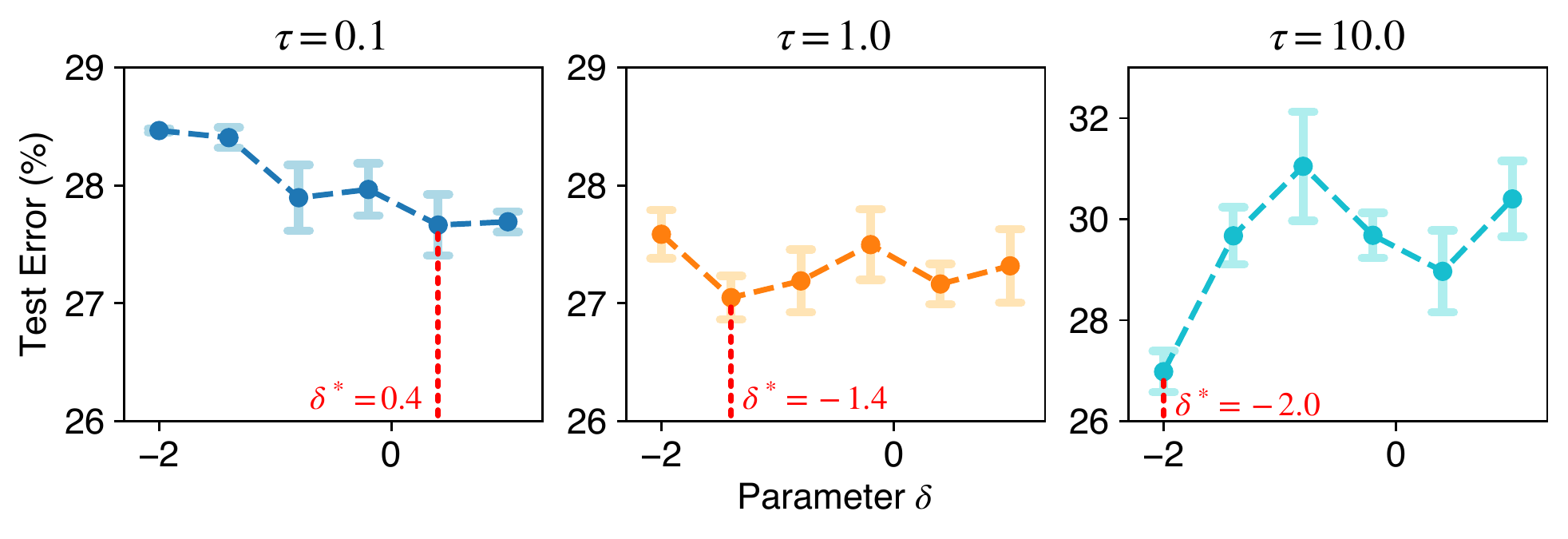}
\caption{The comparison of ensemble test error achieved by our NESBS (SVGD-RD) algorithm with varying $\delta$ under different softmax temperature $\tau$ given an ensemble size of $n=5$. We use $\delta^*$ to denote the optimal $\delta$ we obtained in our SVGD-RD algorithm under different temperature $\tau$. The test error for each $\delta$ is reported with the mean and standard error of five independent trials.}
\label{fig:optimal-diversity}
\end{figure}

\subsection{The Advantages of Controllable Diversity in SVGD-RD}\label{sec:exp-optimal-diversity}

To examine the advantages of controllable diversity in our SVGD-RD, we firstly compare the search effectiveness and efficiency achieved by our NESBS (MC Sampling) and NESBS (SVGD-RD) algorithm with varying softmax temperature $\tau$ (appeared in \eqref{eq:discrete-dist}). A larger temperature $\tau$ will lead to a flatter posterior distribution and hence degenerate its capability of characterizing single-model performances of neural architectures as indicated in \eqref{eq:discrete-dist}. We use these posterior distributions with varying temperature $\tau$ to simulate the possible posterior distributions we may obtain in practice. Figure \ref{fig:robust-tau} illustrates the comparison on CIFAR-10 in the DARTS search space with an ensemble size of $n=5$. Notably, our NESBS (SVGD-RD) with controllable diversity can consistently achieve improved search effectiveness and efficiency than our NESBS (MC Sampling). Interestingly, this improvement becomes larger in the case of $\tau=0.1,10.0$, which should be the consequences of a bad exploration and exploitation achieved by our NESBS (MC Sampling), respectively. These results therefore suggest that the controllable diversity in our SVGD-RD generally can lead to improved search effectiveness and efficiency than our NESBS (MC Sampling).

We further provide the comparison of ensemble test error achieved by our SVGD-RD with varying $\delta$ under different softmax temperature $\tau$ in Figure \ref{fig:optimal-diversity}. Notably, when the posterior distribution tends to be flatter (i.e., $\tau=10$), a smaller $\delta$ is preferred by our SVGD-RD in order to sample architectures with better single-model performances while maintaining the compelling diverse model predictions. Meanwhile, when this posterior distribution tends to be sharper (i.e., $\tau=0.1$), a larger $\delta$ is preferred by our SVGD-RD in order to sample architectures with a larger diverse model predictions while preserving the competitive single-model performances. Based on this controllable diversity and hence the controllable trade-off between the single-model performances and the diverse model predictions, our SVGD-RD is thus capable of achieving comparable performances under varying $\tau$, which usually improve over our NESBS (MC Sampling) by comparing them with the results in Figure~\ref{fig:robust-tau}. These results further validate the advantages of the controllable diversity in our SVGD-RD.
\end{appendices}

\end{document}